%% file: main.tex
\documentclass[12pt]{article}
\usepackage{times}

\input{math_commands.tex}

\usepackage{fancyhdr}
\usepackage{hyperref}
\usepackage{url}
\usepackage[round]{natbib}
\usepackage{amsmath,amssymb,amsfonts,amsthm,mathtools,bm}
\usepackage{booktabs,multirow,array}
\usepackage{graphicx}
\usepackage[dvipsnames]{xcolor}
\usepackage[table]{xcolor}
\usepackage{microtype}
\usepackage{xspace}
\usepackage{enumitem}
\usepackage{algorithm}
\usepackage{algpseudocode}
\usepackage{appendix}
\usepackage{titletoc}
\hypersetup{colorlinks=true,citecolor=NavyBlue,linkcolor=NavyBlue,urlcolor=NavyBlue}
\usepackage{subfig}
\usepackage{xcolor}
\usepackage{threeparttable}
\definecolor{adamwcolor}{HTML}{1FAC59}
\definecolor{muoncolor}{HTML}{007FFF}
\definecolor{manocolor}{HTML}{FFA500}
\definecolor{normpregcolor}{HTML}{CB2DD7}
\definecolor{normprecolor}{HTML}{E44628}
\definecolor{normprebg}{HTML}{FFFAFA}

\theoremstyle{plain} 
\newtheorem{theorem}{Theorem}
\newtheorem{proposition}{Proposition}
\newtheorem{lemma}{Lemma}
\newtheorem{corollary}{Corollary}
\theoremstyle{definition}

\newtheorem{remark}{Remark}
\usepackage{wrapfig}

\newcommand{\muon}{\textup{Muon}\xspace}
\newcommand{\method}{\textup{NormPre}\xspace}
\newcommand{\methodG}{\textup{NormPre-G}\xspace}
\newcommand{\methodL}{\textup{NormPre-L}\xspace}

\newcommand{\orient}{\mathcal{O}}

\titlecontents{section}[3em]{\vspace{-1pt}}%
{\bfseries\contentslabel{2em}}%
{}%
{\titlerule*[0.5pc]{.}\contentspage}%

\titlecontents{subsection}[5em]{\vspace{-1pt}}%
{\contentslabel{2em}}%
{}%
{\titlerule*[0.5pc]{.}\contentspage}%

\titlecontents{subsubsection}[7em]{\vspace{-1pt}}%
{\contentslabel{3em}}%
{}%
{\titlerule*[0.5pc]{.}\contentspage}%

\usepackage[most]{tcolorbox}

\definecolor{claimbg}{HTML}{F0F8FF}
\newtheoremstyle{claimtitle}
  {\topsep}
  {\topsep}
  {\itshape}
  {}
  {\bfseries\itshape}
  {}
  {0em}
  {\thmnote{#3}}

\theoremstyle{claimtitle}
\newtheorem*{claim}{}

\theoremstyle{plain}
\newcommand{\newtcbenvironment}[2]{%
    \tcolorboxenvironment{#1}{%
        #2,
        enhanced,
        breakable,
        sharp corners,
        leftrule=2pt,
        rightrule=0pt,
        toprule=0pt,
        bottomrule=0pt
    }%
    \tcolorboxenvironment{#1*}{%
        #2,
        enhanced,
        breakable,
        sharp corners,
        leftrule=2pt,
        rightrule=0pt,
        toprule=0pt,
        bottomrule=0pt
    }%
}

\newtcbenvironment{claim}{
    colframe=NavyBlue,
    colback=claimbg
}

\date{}

\usepackage{parskip}
\title{\Large{Normalize-Then-Precondition: A Hierarchical Approach to Marginal Scale and Interaction Geometry for LLM Training}}
\author{
  Zixuan Gong\\
  \small{Gaoling School of Artificial Intelligence} \\ 
  \small{Renmin University of China} \\
  \small{\texttt{zxgong@ruc.edu.cn}} 
  \and
  Zeyu Gan\\
  \small{Gaoling School of Artificial Intelligence} \\ 
  \small{Renmin University of China} \\
  \small{\texttt{zygan@ruc.edu.cn}}
  \and
  Jiaye Teng\\
  \small{School of Statistics and Management} \\ 
  \small{Shanghai University of Finance and Economics} \\
  \small{\texttt{tengjiaye@sufe.edu.cn}}
  \and
  Yong Liu\thanks{Corresponding author.}\\
  \small{Gaoling School of Artificial Intelligence} \\ 
  \small{Renmin University of China} \\
  \small{\texttt{liuyonggsai@ruc.edu.cn}}
}

\ifdefined\usebigfont

\usepackage{times}
\usepackage[fontsize=13pt]{scrextend}
\usepackage[left=1.56in,right=1.56in,top=1.6in,bottom=1.74in]{geometry}
\else
\usepackage[margin=0.98in]{geometry}
\fi

\begin{document}

\maketitle

\begin{abstract}
Matrix optimizers have emerged as a promising direction, with Muon standing out as a prominent design.
Revisiting Muon through its full-Gram representation, we observe that it jointly processes marginal-scale and interaction information.
This opens an alternative way to organize geometric information hierarchically, motivating the \emph{\textbf{Normalize-Then-Precondition}} framework.
Specifically, it first uses diagonal-Gram information to construct a marginally normalized update, then applies spectral preconditioning to its directional interaction geometry.
Building on this framework, we develop \textbf{\method} with \methodG and \methodL adopting global and localized spectral preconditioning, grounded in spectral-norm steepest descent and a regularized formulation followed by leading mode selection, respectively.
To enable large-scale training, \methodG uses Newton-Schulz iterations and \methodL employs randomized sketching to approximate the leading interaction eigenspace.
Theoretically, we establish $\mathcal{O}(T^{-1/2})$ convergence guarantees for simplified versions of \method.
Across extensive pretraining experiments on GPT-2 Small, LLaMA and Qwen3, both variants consistently outperform AdamW, Muon and MANO under matched training budgets. 
Further efficiency and spectral analyses reveal the complementary strengths of two variants and characterize their performance-efficiency trade-off.
We open-source our code through a GitHub repository at \href{https://github.com/zx-gong/NormPre}{https://github.com/zx-gong/NormPre}.
\end{abstract}

\newpage
\section{Introduction}
\label{sec:introduction}
Large language model (LLM) training places increasing demands on optimizer performance, efficiency, and scalability. Recent optimizers address these demands by exploiting richer parameter structure, from individual coordinates~\citep{kingma2015iclr-adam, loshchilov2019decoupled,chen2023symbolic,liu2024sophia,liang2026cautious} and parameter blocks~\citep{zhang2025adam,wang2025sharpness} to full matrices~\citep{gupta2018shampoo,vyas2025soap,jordan2024muon,liu2025muon,gu2026mano,deng2026rmnp,xu2026width}. More specifically, orthogonalization and normalization have emerged as prominent matrix-level techniques. On the orthogonalization front, Muon applies a matrix-sign transformation to momentum via Newton-Schulz iterations for rapid convergence~\citep{jordan2024muon, liu2025muon}. Alternatively, other recent methods use row or column normalization to construct matrix updates. For example, MANO combines momentum projection with alternating normalization~\citep{gu2026mano}, and MOGA derives such normalized updates from mean-normalized operator norms under a steepest-descent perspective~\citep{xu2026width}.

Orthogonalization and normalization may appear to follow different routes, yet their Gram representations reveal a shared structure. For a matrix update $X$, the transformations take the forms
$$
\underbrace{\Phi=(\textcolor{muoncolor}{\mathbf{XX^\top}})^{-1/2}X}_{\text{Full-Gram Representation}},\quad
\underbrace{\Psi=\bigl[\textcolor{manocolor}{\textbf{Diag(diag(}\mathbf{XX^\top}\textbf{))}}\bigr]^{-1/2}X}_{\text{Diagonal-Gram Representation}},
$$
where $\operatorname{Diag}(\operatorname{diag}(\cdot))$ retains only the diagonal entries, and the inverse-square-root expressions require full row rank for $\Phi$ and nonzero rows for $\Psi$, respectively. Muon jointly processes marginal scales and cross-row interactions encoded in the full Gram matrix, while normalization draws solely on the Gram diagonal to equalize row norms without altering pairwise angles.
As an extreme example, consider fixed row directions with one row norm of $10^4$ and all others of $10^{-3}$. As this scale gap increases further, the leading Gram eigenvector approaches the largest row's coordinate axis, even if the other rows are strongly correlated. The principal mode thus becomes dominated by marginal scales despite unchanged directional interactions. By removing this scale weighting, diagonal-Gram normalization provides a reference update for processing directional interactions.
Together, these observations open an alternative way to organize the available geometric information: 
\begin{center}
    \textbf{\emph{Can we construct a normalized update from scales and refine it via directional interactions?}} 
\end{center}
While recent methods apply normalization to gradients or momentum to improve the numerical conditioning of orthogonalization~\citep{ma2024swan, chang2026muoneq}, \textit{a systematic framework} for formalizing the hierarchical organization has yet to be established. In particular, it remains unclear \textit{whether the normalized update should undergo a full-spectrum transformation or serve as a reference for more targeted spectral modifications}.

To address this gap, we formalize the hierarchical organization as the \textbf{Normalize-Then-Precondition} framework. The marginally normalized update retains a directional interaction spectrum encoded by $\Gamma=\Psi\Psi^\top$. To exploit this spectrum, one choice is \emph{global spectral preconditioning}, which performs orthogonalization of $\Psi$ and solves a steepest descent problem under the spectral norm constraint. Another choice is to treat the normalized update $\Psi$ as a reference and minimize the changes needed to satisfy the spectral constraint, seeking the closest update within the unit spectral norm ball. This preference is formalized by regularized spectral steepest descent, whose solution clips singular values above one and leaves the others unchanged. A spectral budget $r$ further restricts the transformation to at most $r$ leading super-unit modes, giving \emph{localized spectral preconditioning}. Here, $r$ controls the selected subspace dimension for spectral transformation and the final update remains dense.

\begin{figure}[t]
    \centering
    \subfloat[GPT-2 Small on OWT]{
        \begin{minipage}[t]{0.245\linewidth}
            \centering
            \includegraphics[width=\linewidth]{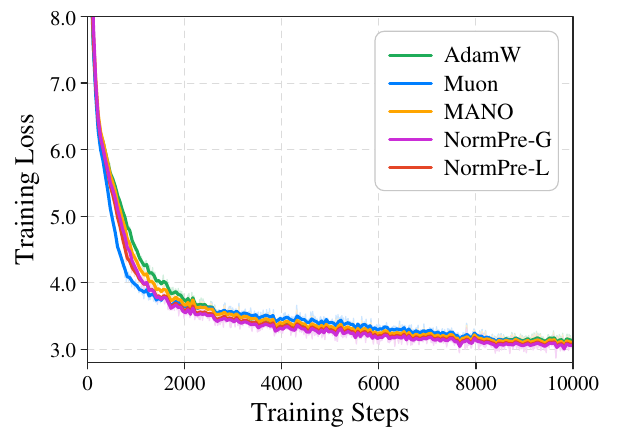}

            \vspace{2pt}

            \includegraphics[width=\linewidth]{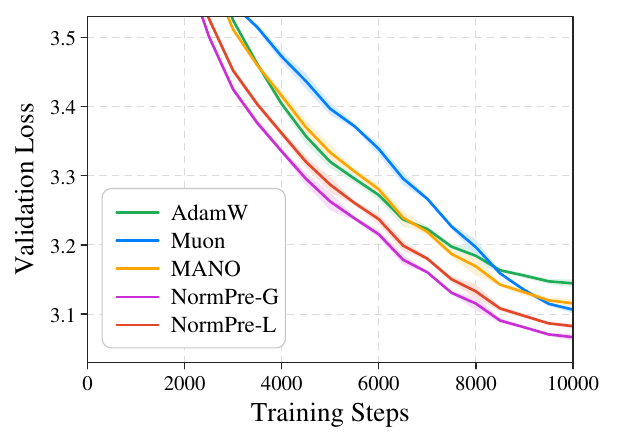}
            \label{fig:gpt2small}
        \end{minipage}
    }\hspace{-6pt}%
    \subfloat[LLaMA-130M on C4]{
        \begin{minipage}[t]{0.245\linewidth}
            \centering
            \includegraphics[width=\linewidth]{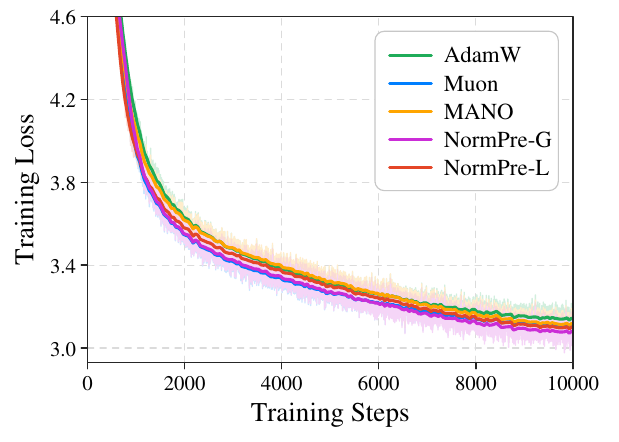}

            \vspace{2pt}

            \includegraphics[width=\linewidth]{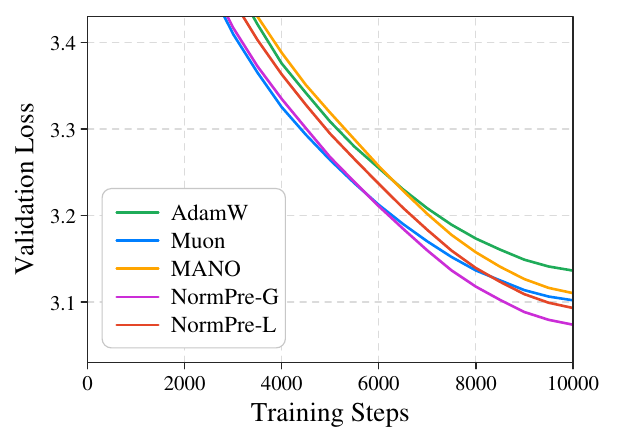}
            \label{fig:llama130m_c4_trainval_loss}
        \end{minipage}
    }\hspace{-6pt}%
    \subfloat[LLaMA-350M on C4]{
        \begin{minipage}[t]{0.245\linewidth}
            \centering
            \includegraphics[width=\linewidth]{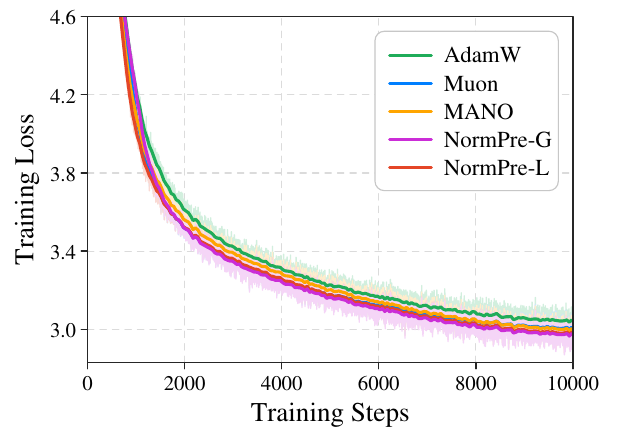}

            \vspace{2pt}

            \includegraphics[width=\linewidth]{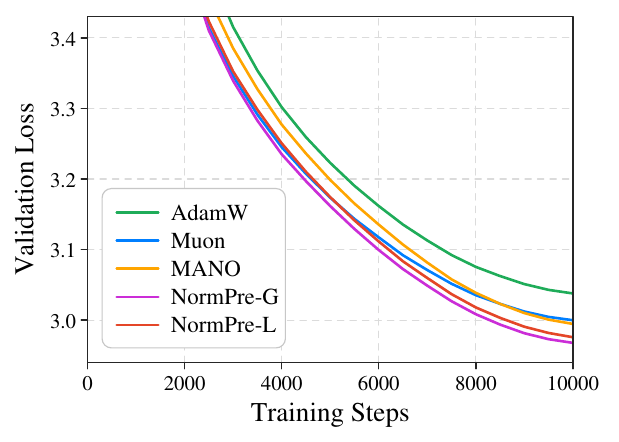}
            \label{fig:llama350m_c4_trainval_loss}
        \end{minipage}
    }\hspace{-6pt}%
    \subfloat[LLaMA-1.3B on C4]{
        \begin{minipage}[t]{0.245\linewidth}
            \centering
            \includegraphics[width=\linewidth]{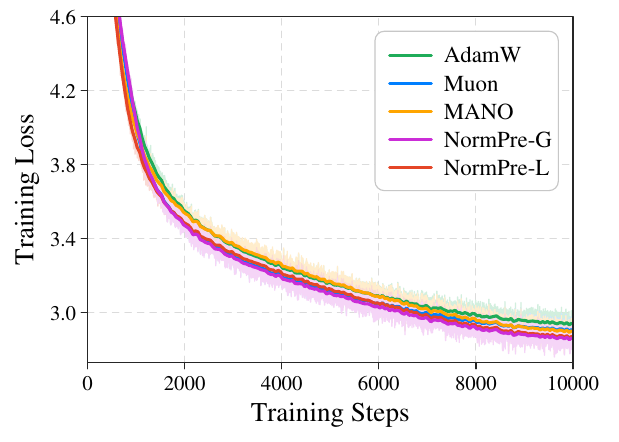}

            \vspace{2pt}

            \includegraphics[width=\linewidth]{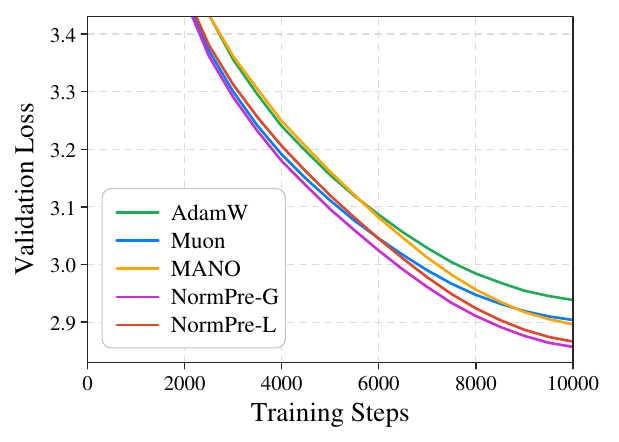}
            \label{fig:llama1p3b_c4_trainval_loss}
        \end{minipage}
    }
    \caption{\textbf{GPT-2 Small on OpenWebText and LLaMA Models on C4.} Training and validation loss for GPT-2 Small and LLaMA-\{130M, 350M, 1.3B\} with AdamW, Muon, MANO, \methodG and \methodL. Dark and light curves denote 50-step moving averages and raw training trajectories.}
    \label{fig:llama_c4}
\end{figure}

Building on this framework, we propose \method, a family of matrix optimizers comprising \methodG and \methodL, which realize marginal normalization followed by global and localized spectral preconditioning, respectively. Both variants combine relaxed tangent momentum processing with alternating row/column normalization and consistent update RMS scaling. Specifically, \methodG approximates the full-spectrum transformation via Newton-Schulz iterations. For \methodL, the leading interaction eigenspace is obtained through exact extraction or scalable randomized sketching. We also establish $\mathcal{O}(T^{-1/2})$ convergence guarantees for simplified \methodG and exact \methodL, extending them to sketch-based \methodL under controlled approximation error.
Empirically, our extensive pretraining evaluations span GPT-2 Small on OpenWebText, LLaMA models up to 1.3B on C4, and Qwen3 models up to 1.7B on Pile. Under matched training budgets, both variants consistently outperform established baselines that represent distinct optimization paradigms: coordinate-wise adaptation (AdamW), matrix orthogonalization (Muon), and normalization (MANO). 
Furthermore, our spectral dynamics and efficiency analyses highlight that \methodL with localized preconditioning delivers strong optimization gains at lower cost and \methodG with global preconditioning further reduces validation loss, providing practitioners with flexible choices to balance optimization performance and computational cost.

\begin{figure}[t]
    \centering
    \includegraphics[width=0.95\linewidth]{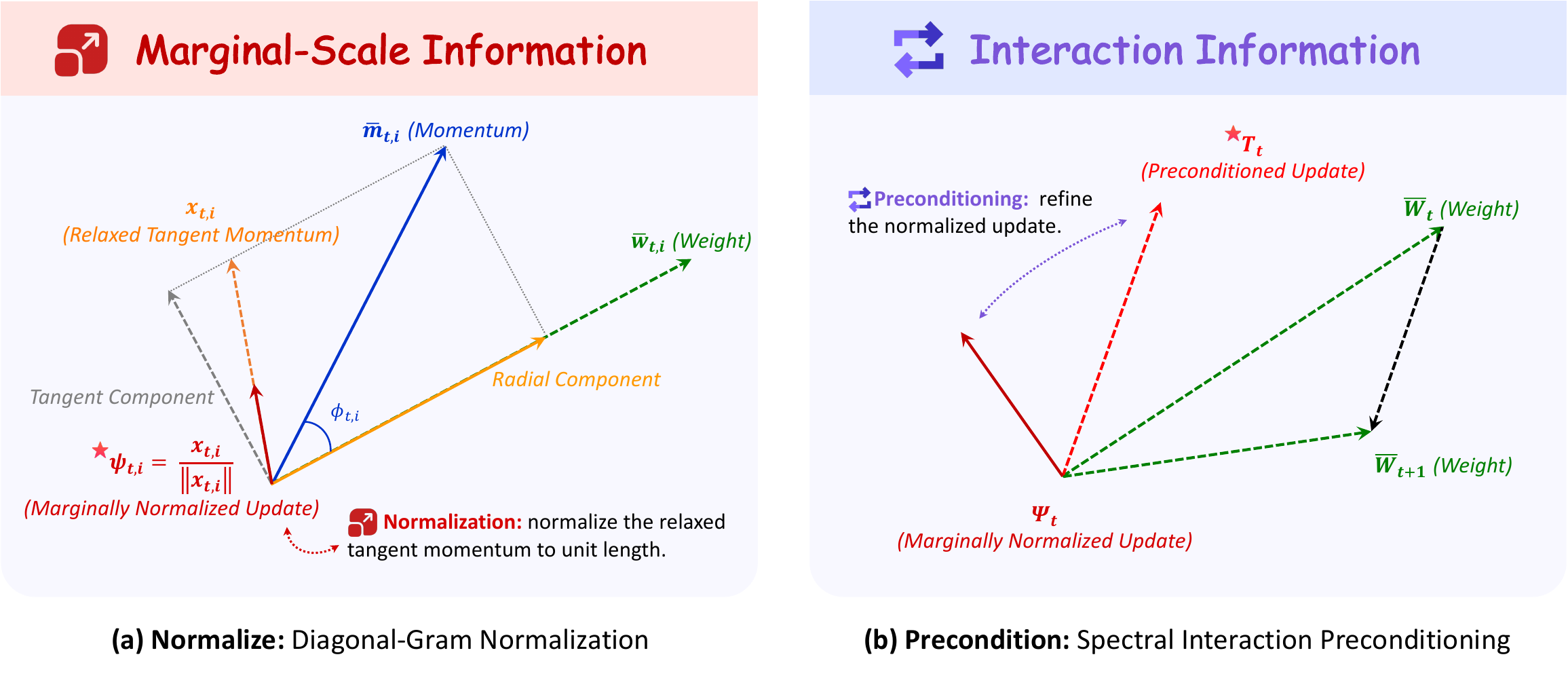}
    \caption{Geometric Workflow of the \method Optimizer (Algorithm~\ref{alg:normpre}).}
    \label{fig:normpre}
\end{figure}

Our main contributions are summarized as follows.

\noindent\textbf{(a) The Normalize-Then-Precondition Framework.}\quad
\textbf{\textit{Conceptually}}, departing from joint geometric transformations, we establish a systematic framework to hierarchically construct a normalized update from marginal scales and refine it via interaction information. Unifying full-spectrum and targeted schemes, we formulate global and localized spectral preconditioning as a solution to spectral-norm steepest descent and a regularized formulation with leading mode selection, respectively (Section~\ref{sec:prelim}).

\noindent\textbf{(b) The NormPre Optimizer.}\quad 
\textbf{\textit{Algorithmically}}, we instantiate this framework as the \method optimizer, featuring alternating row-column normalization and spectral preconditioning. We provide scalable implementations using Newton-Schulz iterations for global preconditioning (\methodG) and randomized sketching for localized eigenspace extraction (\methodL), together with theoretical $\mathcal{O}(T^{-1/2})$ convergence guarantees for simplified variants (Theorem~\ref{thm:normpre-convergence} and Corollary~\ref{cor:normpre-sketch-convergence}).

\noindent\textbf{(c) Optimization Performance and Trade-offs.}\quad
\textbf{\textit{Empirically}}, we demonstrate the superior optimization performance of \method across the evaluated architectures and model scales, consistently outperforming established baselines (AdamW, Muon, MANO). Furthermore, we characterize the performance-efficiency trade-off between its two variants, providing flexible choices to balance optimization gains and computational cost (Section~\ref{sec:experiments}).

\section{Geometric Foundations of Normalize-Then-Precondition}
\label{sec:prelim}
This section formalizes the Normalize-Then-Precondition framework. We begin by establishing a full-Gram view of row-orthogonalized updates (Section~\ref{sec:prelim-muon}) and derive the normalized update with the marginal-scale information encoded in the Gram diagonal (Section~\ref{sec:prelim-normalization}). We finally formulate different spectral preconditioning realizations (Section~\ref{sec:prelim-spectral-preconditioning}).
Throughout this section, $X\in\mathbb R^{m\times n}$ denotes a matrix-valued update signal. We write its thin SVD as $X=U\Sigma V^\top$. We present the row-side analysis and the column-side counterpart follows by applying the same construction to $X^\top$.

\subsection{Orthogonalized Updates and Full-Gram Transformation}\label{sec:prelim-muon}
Muon~\citep{jordan2024muon} orthogonalizes a momentum matrix $X$ by replacing it with the nearest semi-orthogonal matrix. For $X\in\mathbb{R}^{m\times n}$ with $m\leq n$, row orthogonalization is defined by
\begin{align}
    \Phi\in\arg\min_{O}\{\|O-X\|_F: OO^\top=I_m\}.
    \label{eq:muon-semi-orthogonal-objective}
\end{align}
One solution is given by the matrix sign function, $\Phi=\operatorname{msign}(X)= UV^\top.$
Accordingly, the matrix-sign update transforms the singular-value matrix $\Sigma$ into identity, yielding $\Phi\Phi^\top=I_m$. In practice, Muon approximates this orthogonalization efficiently using Newton-Schulz iterations. We next identify an equivalent full-Gram form of the same update, showing how Muon's orthogonalization depends on the full row Gram matrix.

\begin{claim}
    \textbf{Full-Gram Representation of Muon.}
    If $XX^\top\succ0$, the matrix-sign update admits the full-Gram representation
    \begin{equation}
        \Phi=\operatorname{msign}(X)=(XX^\top)^{-1/2}X.
    \label{eq:muon-full-gram-geometry}
    \end{equation}
\end{claim}
Let $x_i^\top$ denote the $i$-th row of $X$. The entries of the full row Gram matrix satisfy $(XX^\top)_{ii}=\|x_i\|_2^2,$ and $(XX^\top)_{ij}=\langle x_i,x_j\rangle$ for $i\neq j$.
Its diagonal entries encode the marginal scales of individual rows, whereas its off-diagonal entries encode their pairwise interactions. Together, these two components determine the inverse-square-root operator in Equation~\ref{eq:muon-full-gram-geometry}, which acts on the entire singular spectrum of $X$. Thus, Muon jointly utilizes marginal-scale and cross-row interaction information through the full row Gram matrix.

\subsection{Diagonal-Gram Normalization}
\label{sec:prelim-normalization}
Within the joint geometry, marginal-scale information is structurally simpler, since each row scale depends only on its corresponding row. This motivates a natural question: \textit{is it possible to first use this simpler marginal-scale information to construct a valid matrix update?}
To answer this question, we isolate the marginal-scale information encoded in the diagonal of $XX^\top$. Since $(XX^\top)_{ii}=\|x_i\|_2^2$, define the diagonal matrix of row scales as
\begin{align*}
    D(X):=\operatorname{Diag}\left(\operatorname{diag}(XX^\top)\right)^{1/2}=\operatorname{Diag}\left(\|x_1\|_2,\ldots,\|x_m\|_2\right),
\end{align*}
where $\operatorname{diag}(\cdot)$ extracts the diagonal entries as a vector and $\operatorname{Diag}(\cdot)$ constructs a diagonal matrix.

\begin{claim}
    \textbf{Diagonal-Gram Representation of Marginal Normalization.}
    If $X$ has no zero rows, the marginally normalized update admits the diagonal-Gram representation
    \begin{equation}
        \Psi=D(X)^{-1}X=\operatorname{Diag}\left(\operatorname{diag}(XX^\top)\right)^{-1/2}X.
    \label{eq:diagonal-gram-transformation}
    \end{equation}
\end{claim}
Equation~\ref{eq:diagonal-gram-transformation} provides the diagonal-Gram counterpart of the Muon transformation in Equation~\ref{eq:muon-full-gram-geometry}. Muon constructs its orthogonalized update from the inverse square root of the full row Gram matrix, while $\Psi$ uses only its diagonal component to normalize the marginal scales. This transformation coincides with the core normalization used by normalization-based optimizers~\citep{gu2026mano,xu2026width,deng2026rmnp,pethick2025training,glentis2025minimalist}. Define the normalized interaction matrix as $\Gamma(X):=\Psi\Psi^\top$. By construction, $\Gamma(X)_{ii}=1$ and its off-diagonal entries $\Gamma(X)_{ij}=\langle x_i,x_j\rangle/(\|x_i\|_2\|x_j\|_2)$ encode the cosine similarities between rows.

\subsection{Spectral Interaction Preconditioning}\label{sec:prelim-spectral-preconditioning}

\textbf{Hierarchical Geometric Construction.}\quad 
Beyond serving as an effective standalone update, the normalization $\Psi$ retains directional correlations in $\Gamma(X)$, providing a natural foundation for further spectral preconditioning. Having processed the marginal scales, we now study two realizations of this preconditioning stage (Figure~\ref{fig:spectral}): \textit{global spectral preconditioning}, which performs full-spectrum transformation, and \textit{localized spectral preconditioning}, which retains $\Psi$ as a reference and concentrates the transformation on selected interaction modes.

Let $\Psi=\widetilde U\widetilde\Sigma\widetilde V^\top$ be a thin SVD with singular values $\widetilde\sigma_1\geq\cdots\geq\widetilde\sigma_m\geq0$. Accordingly, $\Gamma(X)=\Psi\Psi^\top=\widetilde U\Lambda\widetilde U^\top$, where $\Lambda=\widetilde\Sigma^2=\operatorname{Diag}(\lambda_1,\ldots,\lambda_m)$ and $\lambda_i=\widetilde\sigma_i^2$. The columns $\widetilde u_i$ of $\widetilde U$ define the corresponding spectral interaction directions.

\begin{figure}
    \centering
    \includegraphics[width=0.95\linewidth]{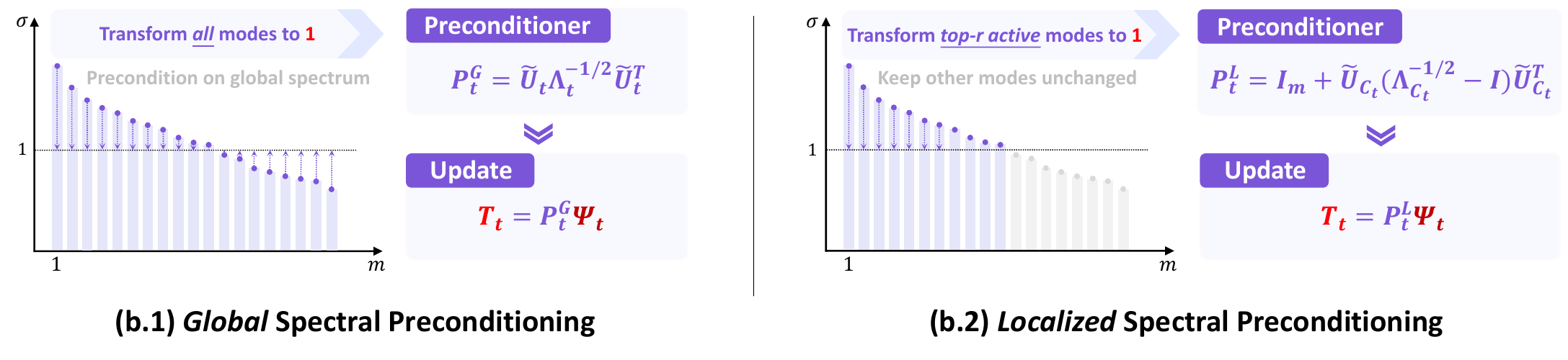}
    \caption{Spectral Interaction Preconditioning.}
    \label{fig:spectral}
\end{figure}

\textbf{Global Spectral Preconditioning.}\quad A direct method is to solve the spectral steepest descent problem for the normalized update $\Psi$~\citep{bernstein2024old},
\begin{equation}
    T^{\mathrm{G}}\in\arg\max_{\|T\|_{\mathrm{op}}\leq1}\langle\Psi,T\rangle_F.
\end{equation}
A solution is the row-orthogonalized update obtained by performing transformation in Equation~\ref{eq:muon-full-gram-geometry}, $T^{\mathrm{G}}=P^{\mathrm{G}}\Psi=\operatorname{msign}(\Psi),$ with the global preconditioner $P^{\mathrm{G}}=\Gamma(X)^{-1/2}=\widetilde U\Lambda^{-1/2}\widetilde U^\top$ when $\Gamma(X)\succ0$. 
Along the $i$-th spectral interaction direction, $P^{\mathrm{G}}$ applies the scaling factor $\lambda_i^{-1/2}=\widetilde\sigma_i^{-1}$, mapping the corresponding singular value $\widetilde\sigma_i$ to one.

\textbf{Localized Spectral Preconditioning.}\quad Another possible method starts from the observation that $\Psi$ already serves as a marginally normalized dense update. This motivates retaining $\Psi$ as the reference and considering the regularized spectral steepest descent problem
\begin{equation}
    T_+:=\arg\min_{\|T\|_{\mathrm{op}}\leq1}\frac{1}{2}\|T-\Psi\|_F^2=\arg\max_{\|T\|_{\mathrm{op}}\leq1}\left\{\langle\Psi,T\rangle_F-\frac{1}{2}\|T\|_F^2\right\}.
\end{equation}
Its solution is $T_+=P_+\Psi$. Defining the active set $\mathcal A:=\{i:\lambda_i>1\}$, the corresponding preconditioner is $P_+=I_m+\widetilde U_{\mathcal A}(\Lambda_{\mathcal A}^{-1/2}-I)\widetilde U_{\mathcal A}^\top$ (see Proposition~\ref{prop:regularized_spectral_sd} in Appendix~\ref{app:proof-regularized-spectral-steepest-descent}). Thus, the solution modifies only the active spectral modes (clipping to one) while leaving the remaining modes unchanged. 
To further localize the transformation, we impose a spectral budget $r$ and let $\mathcal C\subseteq\mathcal A$ index the $\min\{r,|\mathcal A|\}$ largest active eigenvalues (see Proposition~\ref{prop:optimal-rank-constrained-refinement} in Appendix~\ref{app:proof-regularized-spectral-steepest-descent}). The resulting update is $T^{\mathrm{L}}=P^{\mathrm{L}}\Psi,$ with the localized preconditioner $P^{\mathrm{L}}=I_m+\widetilde U_{\mathcal C}(\Lambda_{\mathcal C}^{-1/2}-I)\widetilde U_{\mathcal C}^\top$.

\section{NormPre}\label{sec:method}
This section instantiates the Normalize-Then-Precondition framework as the \method optimizer, with variants \methodG and \methodL using global and localized spectral preconditioning, respectively.

\subsection{The NormPre Optimizer}\label{sec:optimizer}
Let $W_t\in\mathbb{R}^{m\times n}$ denote a matrix-valued parameter with stochastic gradient $G_t=\nabla_{W_t}\mathcal{L}_t$, and let $M_t=\mu M_{t-1}+G_t$ denote its first-order momentum. Following Sections~\ref{sec:prelim-normalization} and~\ref{sec:prelim-spectral-preconditioning}, Algorithm~\ref{alg:normpre} summarizes the \method optimization procedure.

\begin{algorithm}[!h]
    \caption{The NormPre Optimizer}
    \label{alg:normpre}
    \begin{algorithmic}[1]
    \setlength{\baselineskip}{1.05\baselineskip}
    
    \Require Layer weight $W_t\in\mathbb{R}^{m\times n}$, learning rate $\eta_t$, momentum coefficient $\mu$, weight decay coefficient $\lambda_{\mathrm{wd}}$, optimizer variant (\methodG or \methodL).

    \State Initialize $M_0\gets\mathbf{0}$

    \For{each step}
        \State $G_t\gets\nabla_{W_t}\mathcal{L}_t$
        \State $M_t\gets\mu M_{t-1}+G_t$

        \State $k_t\gets t\bmod2$
        \Comment{Alternating Row/Column Orientation}

        \State $X_t\gets\mathcal T_{W_t,k_t}(M_t)$
        \Comment{Relaxed Tangent Momentum}

        \State $\Psi_t\gets D(X_t)^{-1}X_t$
        \Comment{Diagonal-Gram Normalization}

        \If{\ \methodG\ }
            \State $T_t\gets P^\mathrm{G}_t\Psi_t=\operatorname{msign}(\Psi_t)$
        \ElsIf{\ \methodL\ }
            \State $\left(\widetilde U_{\mathcal C_t},\Lambda_{\mathcal C_t}\right)\gets \operatorname{Eig}_{\mathcal C_t}(\Psi_t\Psi_t^\top)$
            \Comment{Eigenspace Extraction}
            \State $T_t\gets P^\mathrm{L}_t\Psi_t=(I+\widetilde U_{\mathcal C_t}(\Lambda_{\mathcal C_t}^{-1/2}-I)\widetilde U_{\mathcal C_t}^\top)\Psi_t$
        \EndIf
        \Comment{Spectral Interaction Preconditioning}
        
        \State $W_{t+1}\gets W_t-\eta_t\left(\mathcal R_{k_t}(T_t)+\lambda_{\mathrm{wd}}W_t\right)$
        \Comment{Consistent Update RMS}
    \EndFor

    \end{algorithmic}
\end{algorithm}

The remaining optimizer-level components in Algorithm~\ref{alg:normpre} specify how the update signal is constructed across matrix orientations and how the final matrix update is rescaled to a consistent RMS.

\textbf{Alternating Row/Column Orientation.}\quad
We alternate the active orientation between rows and columns across optimization steps, allowing the row-side construction in Section~\ref{sec:prelim} to apply symmetrically to both matrix axes. Specifically, let $k_t=t\bmod 2$ with $\orient_0(A)=A$ and $\orient_1(A)=A^\top$, and define $\overline W_t:=\orient_{k_t}(W_t)$ and $\overline M_t:=\orient_{k_t}(M_t)$.

\textbf{Relaxed Tangent Momentum.}\quad
We distinguish the row-wise momentum components orthogonal and parallel to the current parameter direction (\textit{i.e.}, tangent and radial components). These components primarily govern changes in parameter direction and scale, respectively. Let $\bar w_{t,i}$ and $\bar m_{t,i}$ denote the $i$-th rows of $\overline W_t$ and $\overline M_t$, respectively. For each row, we define $x_{t,i}:=\bar m_{t,i}-\langle \bar m_{t,i},\bar w_{t,i}\rangle\bar w_{t,i}.$ This formulation recovers strict tangent projection when $\|\bar w_{t,i}\|_2=1$. For general weights, this relaxed form preserves the tangent component while retaining a norm-dependent radial contribution. Collecting these rows yields the input signal $X_t:=\mathcal T_{W_t,k_t}(M_t)$ for our framework.

\textbf{Consistent Update RMS.}\quad Following the Muon convention~\citep{liu2025muon}, we set the target RMS of matrix updates to $0.2$, facilitating shared hyperparameters and direct comparison with AdamW and Muon. For update $T\in\mathbb{R}^{m\times n}$ with $\operatorname{RMS}(T):=\|T\|_F/\sqrt{mn}$, we define $\mathcal R_k(T):=0.2\cdot\orient_k^{-1}(T)/\operatorname{RMS}(T)$. Here, $\orient_k^{-1}$ maps the update from the active orientation back to the original parameter layout and the rescaling ensures an update RMS of $0.2$ across matrix parameters.

\subsection{Implementation of Spectral Preconditioning}\label{sec:spectral-implementation}
We now detail the spectral preconditioning step in Algorithm~\ref{alg:normpre}, implemented via Newton-Schulz iteration for \methodG and eigenspace extraction for \methodL.

\textbf{\methodG: Newton-Schulz Implementation.}\quad For the global realization, we approximate $T_t=\operatorname{msign}(\Psi_t)$ using the Newton-Schulz iteration adopted by Muon~\citep{jordan2024muon,liu2025muon}. The iteration acts on the normalized update $\Psi_t$ and maps its singular values toward one, realizing full-spectrum preconditioning. We use five Newton-Schulz iterations throughout our experiments.

\textbf{\methodL: Exact and Sketch Implementations.}\quad For the localized realization, the preconditioner depends on a selected eigenspace of the interaction matrix $\Gamma_t=\Psi_t\Psi_t^\top$.

\begin{itemize}[itemsep=0.3em, parsep=0em, topsep=0.3em, partopsep=0em, leftmargin=2em]
    \item \textit{Exact Eigenspace Extraction.}\quad 
    The exact implementation computes the full eigendecomposition of $\Gamma_t$ to identify the up to $r$ largest active eigenvalues ($\lambda_{t,i}>1$). The selected eigenpairs $(\widetilde U_{\mathcal C_t},\Lambda_{\mathcal C_t})$ are then used in the localized preconditioner.

    \item \textit{Sketch-Based Eigenspace Extraction.}\quad For scalable training, we approximate the leading interaction eigenspace with a randomized sketch~\citep{halko2011finding}. The sketch builds a low-dimensional subspace from $\Psi_t$ and applies Rayleigh--Ritz extraction to estimate the leading eigenpairs. Modes with $\widehat\lambda_{t,i}>1$ are retained within the rank budget $r$ and used in the localized preconditioner. The complete procedure is given in Algorithm~\ref{alg:interaction-sketch}, with details in Appendix~\ref{app:spectral-implementation}.
\end{itemize}

\begin{remark}[Computational Complexity]
    \label{rem:computational-complexity}
    \muon and \methodG have the same computational complexity $\mathcal O(mn+qmns)$, where $s=\min\{m,n\}$. Both use the same $q$-step Newton-Schulz transformation ($\mathcal O(qmns)$), while the additional diagonal-Gram normalization in \methodG contributes only $\mathcal O(mn)$.
    For \methodL, the Exact implementation requires $\mathcal O(m^2n+m^3)$ computation to form and decompose the full interaction matrix. The Sketch-based implementation reduces this cost to $\mathcal O\!\left((p+1)mn\ell+(m+n)\ell^2+\ell^3\right)$ by restricting eigenspace extraction to a subspace of dimension $\ell=\min\{m,r+o\}$, ensuring scalability for $\ell\ll m$ (see Appendix~\ref{app-sec:compute_analysis} for detailed analysis).
\end{remark}

In the following, we provide the convergence guarantee for \method.

\begin{theorem}[Convergence of \method without Momentum]
    \label{thm:normpre-convergence}
    Suppose $\mathcal L$ is $L$-smooth and lower bounded by $\mathcal L_{\inf}$ and let $\Delta_0:=\mathcal L(W_0)-\mathcal L_{\inf}$. Under a fixed orientation, let $\phi_{t,i}$ denote the angle between gradient $\bar g_{t,i}$ and weight $\bar w_{t,i}$, and $\phi'_{t,j,i}$ the angle between normalized row $\psi_{t,j}$ and weight $\bar w_{t,i}$. Suppose $\sin\phi_{t,i}\geq\gamma>0$ and $(\sum_j\cos^2\phi'_{t,j,i})^{1/2}\leq\gamma'$ for all $t,i$. Define the maximum radial ratio $\nu_t:=\max_i\|\bar g_{t,i}-x_{t,i}\|_2/\|x_{t,i}\|_2$ and the spectral factor $\chi_t:=\widetilde{\sigma}_{t,1}/\widetilde{\sigma}_{t,m}$ for \methodG (assuming $\Gamma_t\succ0$ for all $t$), or $\chi_t:=\widetilde{\sigma}_{t,1}$ for \methodL (Exact). Assume there exists $\epsilon>0$ such that $\max_{0\leq t\leq T}\nu_t\gamma'\chi_t\leq1-\epsilon$. For any constant $C>0$, choosing $\eta=C/\sqrt{T+1}$ yields
    $$
    \min_{0\leq t\leq T}\|\nabla\mathcal L(W_t)\|_F\leq\frac{1}{\sqrt{T+1}}\left(\frac{\Delta_0\max\{m,n\}^{1/2}}{\epsilon\gamma C}+\frac{LC\max\{m,n\}^{3/2}}{2\epsilon\gamma}\right).
    $$
\end{theorem}

\begin{remark}
    Theorem~\ref{thm:normpre-convergence} analyzes deterministic \methodG and \methodL without momentum, RMS scaling, or weight decay. The condition $\max_{0\leq t\leq T}\nu_t\gamma'\chi_t\leq1-\epsilon$ is sufficient but not necessary for convergence. By bounding the product of the radial ratio ($\nu_t$), geometric non-orthogonality ($\gamma'$), and spectral factor ($\chi_t$), this condition controls the worst-case reduction in the first-order gradient signal arising from the radial interaction in the refined update.
    For any constant $C>0$, choosing $\eta=C/\sqrt{T+1}$ yields an $\mathcal O(T^{-1/2})$ convergence rate. This bound is minimized at $\eta^\star=\sqrt{\frac{2\Delta_0}{L\max\{m,n\}(T+1)}}$, yielding $\min_{0\leq t\leq T}\|\nabla\mathcal L(W_t)\|_F\leq\frac{\max\{m,n\}\sqrt{2L\Delta_0}}{\epsilon\gamma\sqrt{T+1}}$.
    We defer the proof of Theorem~\ref{thm:normpre-convergence} and the Sketch analysis (Corollary~\ref{cor:normpre-sketch-convergence}) to Appendices~\ref{app:proof-exact} and~\ref{app:proof-sketch}, respectively.
\end{remark}

\section{Experiments}
\label{sec:experiments}
This section evaluates \method across diverse pretraining settings (Section~\ref{sec:exp-performance}), analyzes its efficiency and spectral dynamics (Section~\ref{sec:efficiency}) and characterizes the performance-efficiency trade-off between \methodG and \methodL (Section~\ref{sec:G-L-comparision}).

\subsection{Scaling Experiments}\label{sec:exp-performance}
\label{sec:scaling}

\textbf{Experimental Setup.}\quad We evaluate \method across model scales, architectures and datasets. Our experiments cover GPT-2 Small on OpenWebText~\citep{radford2019language,Gokaslan2019OpenWeb}, LLaMA-\{130M, 350M, 1.3B\} on C4~\citep{touvron2023llama,raffel2020exploring}, and Qwen3-\{0.6B, 1.7B\} on Pile~\citep{yang2025qwen3,gao2020pile}. 
We compare \methodG and \methodL with AdamW~\citep{loshchilov2019decoupled}, Muon~\citep{jordan2024muon} and MANO~\citep{gu2026mano}.
\methodG uses five Newton-Schulz iterations and \methodL uses the Sketch implementation with rank $r=32$. All models are trained for $10{,}000$ steps with a sequence length of $1024$ and an effective batch size of $512$. Within each setting, all optimizers use matched model initialization, data order, training budget and evaluation protocol. Detailed configurations are provided in Appendix~\ref{app:experimental-configurations}.

\begin{table}[t]
    \centering
    \caption{\textbf{Validation loss across model scales, architectures and pretraining datasets.} GPT-2 Small results show mean $\pm$ standard deviation over three matched seeds, and larger models report single-run results under matched training budgets. Shaded columns indicate our variants. The best baseline and our outperforming variants are shown in \underline{underline} and \textbf{bold}, respectively.}
    \label{tab:scaling-summary}
    \small
    \renewcommand{\arraystretch}{1.1}
    \resizebox{\linewidth}{!}{%
    \begin{tabular}{lcccc>{\columncolor{normprebg}}c>{\columncolor{normprebg}}c}
        \toprule[1.5pt]
        \multicolumn{2}{c}{Setting} & \multicolumn{5}{c}{Validation Loss $\downarrow$} \\
        \cmidrule(lr){1-2}\cmidrule(lr){3-7}
        Model & Dataset &
        \textcolor{adamwcolor}{\textbf{AdamW}} &
        \textcolor{muoncolor}{\textbf{Muon}} &
        \textcolor{manocolor}{\textbf{MANO}} &
        \multicolumn{1}{c}{\textcolor{normpregcolor}{\textbf{\methodG}}} &
        \multicolumn{1}{c}{\textcolor{normprecolor}{\textbf{\methodL}}} \\
        \midrule
        GPT-2 Small & OpenWebText
        & 3.1444 {\scriptsize$\pm$0.0054}
        & \underline{3.1064 {\scriptsize$\pm$0.0054}}
        & 3.1156 {\scriptsize$\pm$0.0057}
        & \textbf{3.0667} {\scriptsize$\boldsymbol{\pm}$\textbf{0.0049}}
        & \textbf{3.0826 {\scriptsize$\pm$0.0034}} \\
        LLaMA-130M & C4
        & 3.1363
        & \underline{3.1019}
        & 3.1102
        & \textbf{3.0737}
        & \textbf{3.0930} \\
        LLaMA-350M & C4
        & 3.0378
        & 2.9999
        & \underline{2.9946}
        & \textbf{2.9678}
        & \textbf{2.9758} \\
        LLaMA-1.3B & C4
        & 2.9385
        & 2.9037
        & \underline{2.8963}
        & \textbf{2.8571}
        & \textbf{2.8662} \\
        Qwen3-0.6B & Pile
        & 2.8956
        & \underline{2.8335}
        & 2.8382
        & \textbf{2.7967}
        & \textbf{2.8066} \\
        Qwen3-1.7B & Pile
        & 2.6758
        & 2.6408
        & \underline{2.6205}
        & \textbf{2.5868}
        & \textbf{2.5897} \\
        \bottomrule[1.5pt]
    \end{tabular}%
    }
\end{table}

\begin{table}[t]
    \centering
    \caption{\textbf{Training Efficiency Comparisons.}
    Measurements use the same training configurations as the main experiments and are averaged over 100 optimizer steps after 20 warmup steps.}
    \label{tab:training-efficiency}
    \small
    \renewcommand{\arraystretch}{1}
    \setlength{\tabcolsep}{5pt}

    \resizebox{\linewidth}{!}{%
    \begin{tabular}{ccccccc}
        \toprule[1.5pt]
        Model
        & Dataset
        & Optimizer
        & Optimizer Latency $\downarrow$
        & E2E Step Time $\downarrow$
        & Throughput $\uparrow$
        & Peak Memory $\downarrow$ \\
        &
        &
        & (ms/step)
        & (ms/step)
        & (k tokens/s)
        & (GiB) \\
        \midrule

        \multirow{5}{*}{GPT-2 Small}
        & \multirow{5}{*}{OpenWebText}
        & \textcolor{adamwcolor}{\textbf{AdamW}}
        & 2.7
        & 3364.8
        & 155.82
        & 14.93 \\

        &
        & \textcolor{muoncolor}{\textbf{Muon}}
        & 112.8
        & 3479.8
        & 150.67
        & 14.61 \\

        &
        & \textcolor{manocolor}{\textbf{MANO}}
        & 8.8
        & 3365.6
        & 155.78
        & 14.61 \\

        &
        & \textcolor{normpregcolor}{\textbf{\methodG}}
        & 119.3\rlap{\;\textcolor{muoncolor}{\scriptsize$\uparrow$5.76\%}}
        & 3495.0\rlap{\;\textcolor{muoncolor}{\scriptsize$\uparrow$0.44\%}}
        & 150.01\rlap{\;\textcolor{muoncolor}{\scriptsize$\downarrow$0.44\%}}
        & 14.62 \\

        &
        & \textcolor{normprecolor}{\textbf{\methodL}}
        & 69.8\rlap{\;\textcolor{muoncolor}{\scriptsize$\downarrow$38.12\%}}
        & 3439.4\rlap{\;\textcolor{muoncolor}{\scriptsize$\downarrow$1.16\%}}
        & 152.44\rlap{\;\textcolor{muoncolor}{\scriptsize$\uparrow$1.17\%}}
        & 14.61 \\

        \specialrule{0.6pt}{4pt}{1pt}
        \specialrule{0.6pt}{0pt}{4pt}

        \multirow{5}{*}{LLaMA-1.3B}
        & \multirow{5}{*}{C4}
        & \textcolor{adamwcolor}{\textbf{AdamW}}
        & 25.1
        & 22121.0
        & 23.70
        & 42.34 \\

        &
        & \textcolor{muoncolor}{\textbf{Muon}}
        & 1010.1
        & 23124.2
        & 22.67
        & 37.84 \\

        &
        & \textcolor{manocolor}{\textbf{MANO}}
        & 87.1
        & 22210.3
        & 23.61
        & 37.84 \\

        &
        & \textcolor{normpregcolor}{\textbf{\methodG}}
        & 1054.2\rlap{\;\textcolor{muoncolor}{\scriptsize$\uparrow$4.37\%}}
        & 23141.8\rlap{\;\textcolor{muoncolor}{\scriptsize$\uparrow$0.08\%}}
        & 22.66\rlap{\;\textcolor{muoncolor}{\scriptsize$\downarrow$0.04\%}}
        & 37.83 \\

        &
        & \textcolor{normprecolor}{\textbf{\methodL}}
        & 333.3\rlap{\;\textcolor{muoncolor}{\scriptsize$\downarrow$67.00\%}}
        & 22399.7\rlap{\;\textcolor{muoncolor}{\scriptsize$\downarrow$3.13\%}}
        & 23.41\rlap{\;\textcolor{muoncolor}{\scriptsize$\uparrow$3.26\%}}
        & 37.84 \\
        \bottomrule[1.5pt]
    \end{tabular}%
    }
\end{table}

\textbf{GPT-2 Small Pretraining on OpenWebText.}\quad
Figure~\ref{fig:gpt2small} shows training and validation loss trajectories averaged over three matched seeds. Both variants achieve lower training and validation loss than the baselines during later training, demonstrating the effectiveness of \method. Table~\ref{tab:scaling-summary} confirms that both variants outperform all three baselines in mean best validation loss, with \methodG achieving the lowest value. Compared with the strongest baseline Muon, \methodG and \methodL reduce validation loss by $0.0397$ and $0.0238$, respectively.

\textbf{LLaMA Pretraining on C4 and Qwen3 Pretraining on Pile.}\quad
We extend the comparison to LLaMA-\{130M, 350M, 1.3B\} on C4 and Qwen3-\{0.6B, 1.7B\} on Pile. Table~\ref{tab:scaling-summary} and Figures~\ref{fig:llama130m_c4_trainval_loss}--\ref{fig:llama1p3b_c4_trainval_loss} show that both variants consistently outperform all three baselines across these diverse architectures, with \methodG achieving the lowest loss. For LLaMA, the advantages over the strongest baseline are most pronounced at the largest evaluated scale (1.3B), where \methodG and \methodL reduce validation loss by $0.0392$ and $0.0301$, respectively. These optimization gains extend similarly to Qwen3. Notably, \methodL performs closely to \methodG at the 1.7B scale with a minimal loss difference of $0.0029$ (Qwen3 loss curves are deferred to Appendix~\ref{app-sec:exp-res}).

\subsection{Training Efficiency and Spectral Dynamics}\label{sec:efficiency}
\textbf{Training Efficiency.}\quad Table~\ref{tab:training-efficiency} reports training efficiency on GPT-2 Small and LLaMA-1.3B under the same configurations as the main experiments. 
As summarized in Table~\ref{tab:optimizer-complexity-summary} and Remark~\ref{rem:computational-complexity} (Appendix~\ref{app-sec:compute_analysis}), \methodG shares Muon's asymptotic complexity, but its additional normalization increases optimizer latency by $5.76\%$ and $4.37\%$ relative to Muon on the two models, respectively. \methodL approximates the interaction subspace via sketching, reducing optimizer latency by $38.12\%$ and $67.00\%$. 
These optimizer-level differences have a smaller effect on end-to-end training efficiency because optimizer updates account for only a small fraction of step time, which is largely spent on forward and backward computation. Relative to Muon, \methodG increases step time by about $0.44\%$ and $0.08\%$ on both models, while \methodL reduces it by $1.16\%$ and $3.13\%$, with corresponding improvements in throughput. 
Both variants remain slower per step than AdamW and MANO due to the additional spectral processing, but match the whole-training peak memory of Muon and MANO. Further discussions and efficiency results on Qwen3 are provided in Appendix~\ref{app-sec:exp-res}.

\textbf{Spectral Dynamics.}\quad
Figures~\ref{fig:before_and_after_spectrum_selected}--\ref{fig:row_spectral_transformation} validate our hierarchical Normalize-Then-Precondition design, revealing two key patterns. \textit{(a) Normalization yields an anisotropic base update.} Equalizing row norms removes scale weighting in the Gram matrix, giving smaller-norm rows greater relative influence. This redistributes the spectral mass and might raise eigenvalues across multiple ranks (dotted curves over solid curves). Crucially, this anisotropy persists throughout training, necessitating further refinement. Table~\ref{tab:interaction-source-ablation} confirms that constructing the preconditioner from $\Psi$ rather than $X$ yields lower losses. \textit{(b) Localized preconditioning captures the dominant transformation energy.} Building on the normalized base, global preconditioning contracts modes above one and amplifies positive modes below one. Decomposing its spectral transformation energy (squared singular-value change $\sum_{\lambda_i>0}(\sqrt{\lambda_i}-1)^2$) reveals that the top-32 modes ($\lambda_i>1$) consistently account for approximately $62\%$--$67\%$ across checkpoints. By concentrating strictly on these modes, localized preconditioning covers a substantial portion of the global transformation to yield the optimization gains observed in Figure~\ref{fig:llama_c4}. Further details and training dynamics are provided in Appendix~\ref{app-sec:dynamics}.

\begin{figure}[t]
    \centering
    \subfloat[]{
        \includegraphics[width=0.32\linewidth]{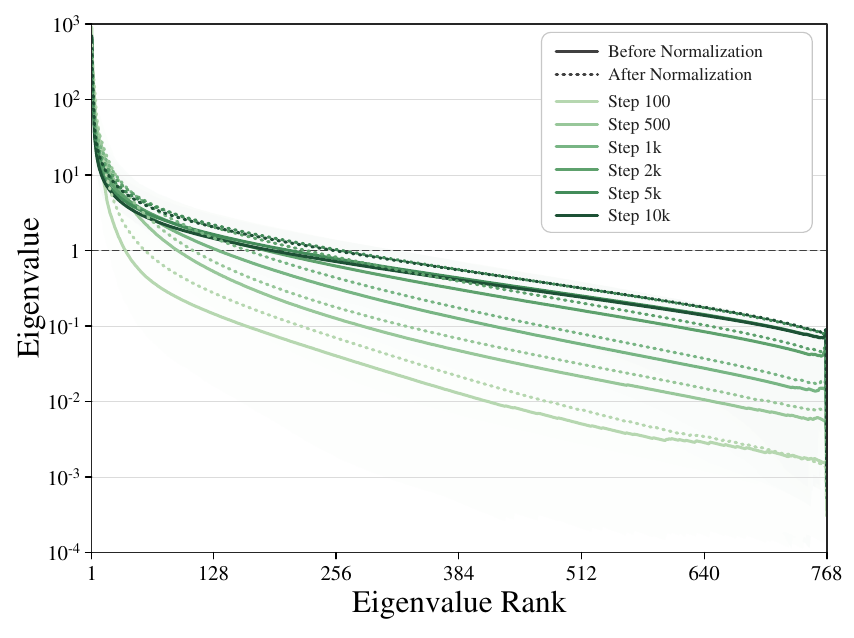}
        \label{fig:before_and_after_spectrum_selected}
    }
    \subfloat[]{
        \includegraphics[width=0.32\linewidth]{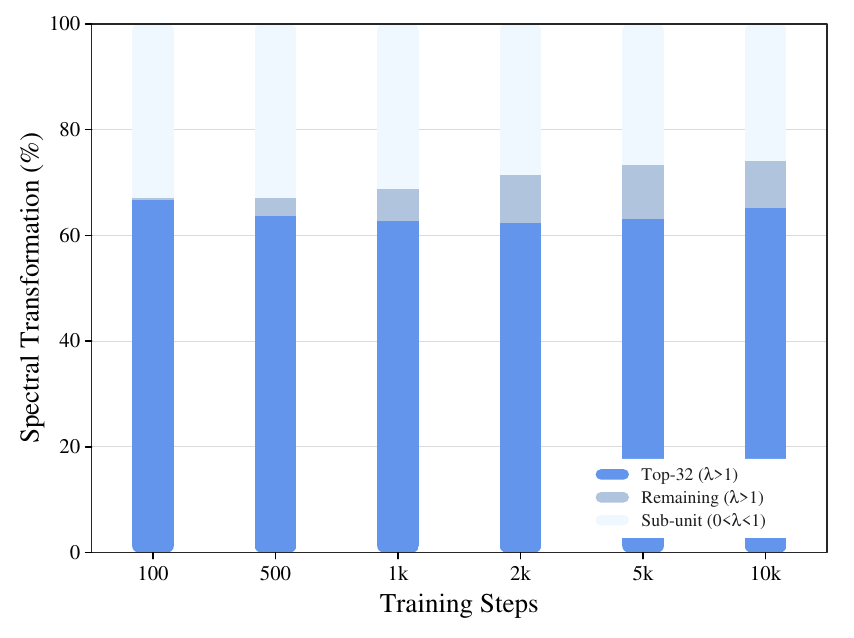}
        \label{fig:row_spectral_transformation}
    }
    \subfloat[]{
        \includegraphics[width=0.32\linewidth]{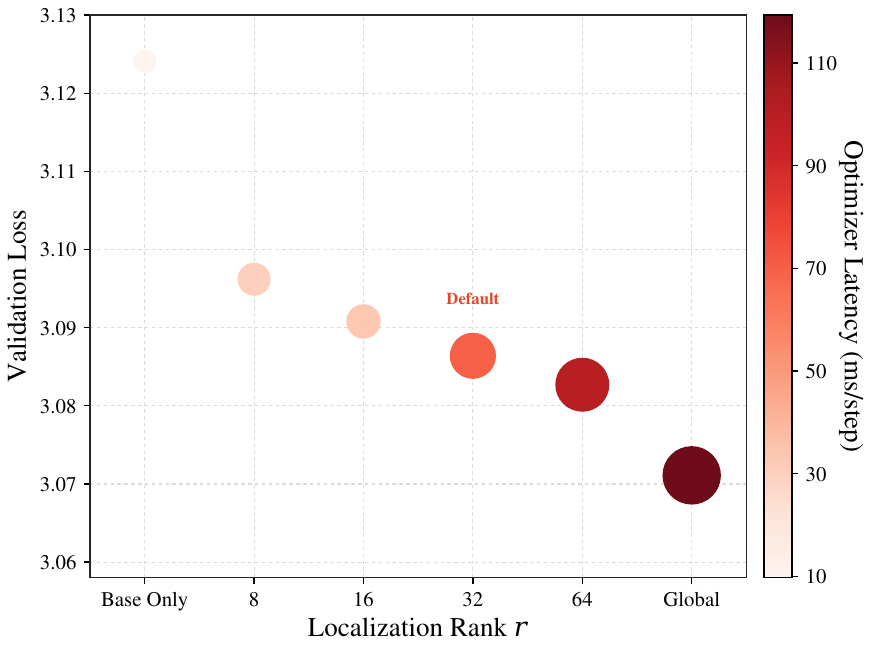}
        \label{fig:gpt2small_rank_localization_frontier_new}
    }
    \caption{(a) Eigenspectra before and after marginal normalization; (b) Spectral transformation percentages under global and localized preconditioning; (c) Performance-efficiency trade-off.}
    \label{fig:spectral_dynamics}
\end{figure}

\subsection{Comparisons of \methodG and \methodL}\label{sec:G-L-comparision}
Drawing on the preceding findings, we discuss the complementary strengths of two variants.

\textit{(a) NormPre-G yields superior performance via full-spectrum transformations.} 
By operating on the entire eigenspectrum, \methodG consistently achieves the lowest validation loss across all evaluations. This indicates that transforming the long-tail modes alongside the dominant spectral energy provides crucial optimization benefits that push the absolute performance limit.

\textit{(b) NormPre-L offers a performance-efficiency trade-off via local-spectrum transformations.}
The normalized base update motivates a regularized steepest-descent problem limiting deviations from this reference, giving \methodL a principled formulation (Propositions~\ref{prop:regularized_spectral_sd} and~\ref{prop:optimal-rank-constrained-refinement}). 
Although empirical results favor the broader transformation of \methodG in validation loss, Figure~\ref{fig:gpt2small_rank_localization_frontier_new} illustrates a clear trade-off on GPT-2 Small that larger $r$ lowers validation loss at higher optimizer cost.
We use $r=32$ as the default spectral budget to balance performance and efficiency. At this budget, \methodL yields clear optimization gains with lower optimizer latency and end-to-end step time than \methodG (Table~\ref{tab:training-efficiency}), serving as a practical and scalable alternative.

\section{Related Work}
\label{sec:related}
\textbf{LLM Optimizers.}\quad Recent optimizers span coordinate-wise methods such as AdamW~\citep{kingma2015iclr-adam,loshchilov2019decoupled}, Lion~\citep{chen2023symbolic}, Sophia~\citep{liu2024sophia}, and Cautious Optimizers~\citep{liang2026cautious}, and block-wise approaches like Adam-mini~\citep{zhang2025adam} and Blockwise LR~\citep{wang2025sharpness}. Alongside these, matrix-level methods like K-FAC~\citep{martens2015optimizing}, Shampoo~\citep{gupta2018shampoo,shi2023distributed}, and SOAP~\citep{vyas2025soap} construct preconditioners from curvature or accumulated statistics. We focus on matrix optimizers operating directly on update geometry via orthogonalization or row/column normalization.

\textbf{Orthogonalization and Normalization.}\quad Matrix structure is increasingly exploited via orthogonalization or lightweight row/column normalization. On the orthogonalization front, Muon applies a matrix-sign transformation to momentum via Newton-Schulz iterations~\citep{jordan2024muon,liu2025muon}, while broader norm-constrained frameworks connect these updates to steepest descent and spectral spheres~\citep{bernstein2024old,pethick:hal-04941364,xie2026controlled}. Alternatively, row/column normalization offers a lightweight way. For example, MANO combines momentum projection with alternating row/column normalization~\citep{gu2026mano}, while RMNP uses only row normalization~\citep{deng2026rmnp}. MOGA derives such normalized updates from mean-normalized operator norms under the steepest-descent perspective~\citep{xu2026width}.

\textbf{Closely Relevant Work.}\quad Several methods retain orthogonalization as a core operation and incorporate normalization or adaptive scaling. One line modifies the post-orthogonalization update: Muon+ applies row/column normalization to the orthogonalized output~\citep{zhang2026muon+}, while NorMuon and AdaMuon introduce neuron-wise and element-wise adaptive scaling~\citep{li2025normuon,si2025adamuon}. 
Another line applies pre-orthogonalization normalization for numerical stability. For instance, SWAN normalizes instantaneous gradients before whitening~\citep{ma2024swan} and MuonEq~\citep{chang2026muoneq} equilibrates momentum to improve Newton-Schulz conditioning.
While these approaches discuss the utility of incorporating normalization, it remains underexplored to formulate a systematic hierarchical framework and investigate if targeted spectral schemes can be effectively developed.
We establish the \textbf{Normalize-Then-Precondition} framework, introducing two realizations of spectral preconditioning from spectral steepest descent and its regularized form with leading mode selection.

\section{Conclusion}
This paper formalizes the Normalize-Then-Precondition framework to hierarchically organize marginal scales and interactions in matrix optimizers. Building on this, we propose \method, which achieves superior LLM pretraining performance over established baselines with a clear performance-efficiency trade-off. We highlight three promising directions for future research addressing current limitations: (a) Developing hardware-aware implementations to improve the efficiency of spectral transformations; (b) Scaling our empirical validation beyond 1.7B parameters; (c) Exploring adaptive spectral schemes to narrow the performance gap between localized and global preconditioning while preserving efficiency.
Overall, our work provides a systematic framework for integrating normalization and spectral preconditioning, inspiring more effective and efficient LLM training.

\subsection*{AI use statement}
We used OpenAI Codex and Anthropic Claude Code to assist with core code implementation, data processing, experimental evaluation, translation, and the verification and refinement of mathematical proofs. Additionally, we used these tools for literature search, manuscript drafting and revision. The research questions and ideas, conceptual framework, methodology, proof strategies and initial proof drafts, and experimental design were developed by the authors without AI assistance. We have reviewed all AI-assisted work, including manuscript text, source code, data processing, and training and evaluation pipelines. We retain full responsibility for the final content of this work, including text, claims or artifacts produced with the aid of generative AI.

\subsection*{Ethics statement}
This work studies effective and efficient optimization methods for large language model training. It does not involve human-subject experiments, collection of personal information or deployment in real-world decision-making systems. Our experiments use established language-modeling datasets, including OpenWebText, C4 and Pile, and we do not introduce or release new training data. We follow standard research practices for experimental evaluation, reporting and reproducibility.

\subsection*{Reproducibility statement}
Section~\ref{sec:method} provides a detailed description of the proposed optimizer \method, including two variants \methodG and \methodL. Section~\ref{sec:experiments} presents the experimental setup, datasets, evaluation protocol and main results, with additional configurations, results, dynamics analyses and ablations provided in Appendix~\ref{app:experimental-details}. Implementation details for spectral preconditioning and computational complexity analyses are given in Appendices~\ref{app:spectral-implementation} and~\ref{app-sec:compute_analysis}. The assumptions, derivations and complete proofs supporting our theoretical results are provided in Appendix~\ref{app:proof}. A repository containing the source code and materials for reproducing the experiments is linked in the abstract.

\bibliography{reference}
\bibliographystyle{unsrtnat}

\newpage
\appendix
\renewcommand{\appendixpagename}{\centering \LARGE Appendix}
\appendixpage
\startcontents[section]
\printcontents[section]{l}{1}{\setcounter{tocdepth}{3}}
\input{Appendix/appendix_exp}

\input{Appendix/appendix_implementation_complexity}
\input{Appendix/appendix_prop}
\input{Appendix/appendix_convergence}

\end{document}

%% file: math_commands.tex
\usepackage{amsmath,amsfonts,bm}

\def\eqref#1{equation~\ref{#1}}

\def\1{\bm{1}}

\DeclareMathAlphabet{\mathsfit}{\encodingdefault}{\sfdefault}{m}{sl}
\SetMathAlphabet{\mathsfit}{bold}{\encodingdefault}{\sfdefault}{bx}{n}

%% file: Appendix/appendix_exp.tex
\newpage

\section{Experimental Details}
\label{app:experimental-details}

\subsection{Experimental Configurations}\label{app:experimental-configurations}

\textbf{Models and Data.}\quad
Our experiments use GPT-2 Small on OpenWebText~\citep{radford2019language,Gokaslan2019OpenWeb}, LLaMA-\{130M, 350M, 1.3B\} on C4~\citep{touvron2023llama,raffel2020exploring} and Qwen3-\{0.6B, 1.7B\} on Pile~\citep{yang2025qwen3,gao2020pile}. GPT-2 Small is reproduced using the nanoGPT codebase~\citep{Karpathy2022}. We follow the general pretraining setup in~\citet{zhao2024galore,raffel2020exploring}. Table~\ref{tab:model-training-configurations} summarizes the model and training configurations.

OpenWebText is processed with the GPT-2 BPE tokenizer, and the preprocessed training and validation data are loaded into CPU memory before training. For C4, we tokenize the English corpus with the T5-base tokenizer. For Pile, we recover documents from the Pythia/Pile token stream and re-tokenize them with the Qwen3 tokenizer. For C4 and Pile, we construct fixed preprocessed datasets of $5{,}120{,}000$ training sequences and $2{,}048$ validation sequences, each with length $1024$. Documents are processed independently by truncating long documents and padding shorter ones without cross-document packing. The resulting data order is fixed and shared across all optimizer runs. With sequence length $1024$, effective batch size $512$ and $10{,}000$ optimization steps, each training trajectory processes approximately $5.24$B nominal token positions.

\begin{table}[!ht]
    \centering
    \caption{Model and Training Configurations.}
    \label{tab:model-training-configurations}
    \small
    \renewcommand{\arraystretch}{1.2}
    \resizebox{\linewidth}{!}{
    \begin{tabular}{lccccccccccc}
    \toprule[1.5pt]
    Model & Model Size & Dataset & Layers & Hidden & FFN & Heads & Seq. Len. & Eff. Batch & Base Peak LR & Warmup & Steps \\
    \midrule
    GPT-2 Small & 124M & OpenWebText & 12 & 768 & 3,072 & 12 & 1,024 & 512 & $6\times10^{-4}$ & 2,000 & 10,000 \\
    LLaMA-130M & 130M & C4 & 12 & 768 & 2,048 & 12 & 1,024 & 512 & $6\times10^{-4}$ & 1,000 & 10,000 \\
    LLaMA-350M & 350M & C4 & 24 & 1,024 & 2,736 & 16 & 1,024 & 512 & $3\times10^{-4}$ & 1,000 & 10,000 \\
    LLaMA-1.3B & 1.3B & C4 & 24 & 2,048 & 5,461 & 32 & 1,024 & 512 & $3\times10^{-4}$ & 1,000 & 10,000 \\
    Qwen3-0.6B & 0.6B & Pile & 28 & 1,024 & 3,072 & 16 & 1,024 & 512 & $3\times10^{-4}$ & 1,000 & 10,000 \\
    Qwen3-1.7B & 1.7B & Pile & 28 & 2,048 & 6,144 & 16 & 1,024 & 512 & $3\times10^{-4}$ & 1,000 & 10,000 \\
    \toprule[1.5pt]
    \end{tabular}}
\end{table}

\textbf{Optimizers.}\quad
We compare \methodG and \methodL with AdamW~\citep{loshchilov2019decoupled}, Muon~\citep{jordan2024muon} and MANO~\citep{gu2026mano}. Table~\ref{tab:optimizer-configurations} summarizes the default optimizer configurations. We use $(\beta_1,\beta_2)=(0.9,0.95)$ for AdamW and momentum coefficient $\mu=0.95$ for Muon, MANO and both \method variants. Muon and \methodG use five Newton-Schulz iterations. \methodL uses the Sketch-based implementation with interaction rank $r=32$, sketch oversampling $o=8$ and one power iteration $p=1$. Both \method variants match the post-preconditioning update RMS to a target value of $0.2$ before the parameter update. For GPT-2 Small, Muon follows the Keller--Jordan convention~\citep{jordan2024muon} for hidden matrix parameters, using an effective peak matrix learning rate of $0.02$ together with width-dependent update scaling. For LLaMA and Qwen3, Muon follows the scalable RMS-matching convention~\citep{liu2025muon} and uses the base learning rate of the corresponding model setting. Muon, MANO and both \method variants are applied to hidden two-dimensional attention and MLP weight matrices, with token embeddings, output heads, normalization parameters, biases and other one-dimensional parameters optimized by AdamW.

\begin{table}[!ht]
    \centering
    \caption{Default Optimizer Configurations. A dash indicates that the corresponding hyperparameter is not applicable.}
    \label{tab:optimizer-configurations}
    \small
    \renewcommand{\arraystretch}{1}
    \resizebox{0.68\linewidth}{!}{
    \begin{tabular}{lccccc}
    \toprule[1.5pt]
    Hyperparameter & AdamW & Muon & MANO & \methodG & \methodL \\
    \midrule
    $\beta_1$ & 0.9 & -- & -- & -- & -- \\
    $\beta_2$ & 0.95 & -- & -- & -- & -- \\
    Momentum $\mu$ & -- & 0.95 & 0.95 & 0.95 & 0.95 \\
    Newton-Schulz steps & -- & 5 & -- & 5 & -- \\
    Weight decay & 0.1 & 0.1 & 0.1 & 0.1 & 0.1 \\
    Interaction rank $r$ & -- & -- & -- & -- & 32 \\
    Sketch oversampling $o$ & -- & -- & -- & -- & 8 \\
    Power iterations $p$ & -- & -- & -- & -- & 1 \\
    \bottomrule[1.5pt]
    \end{tabular}
    }
\end{table}

\textbf{Training.}\quad All models are trained for $10{,}000$ optimization steps with sequence length $1024$ and effective batch size $512$. We use linear warmup followed by cosine decay to $10\%$ of the peak learning rate, with the model-specific learning rates and warmup steps given in Table~\ref{tab:model-training-configurations}. Weight decay is $0.1$ and gradients are clipped at $1.0$. All models are evaluated every $500$ optimization steps. GPT-2 Small uses $200$ randomly sampled validation mini-batches per evaluation, while LLaMA and Qwen3 evaluate on the complete fixed validation split of $2{,}048$ sequences. Within each setting, all compared optimizers use matched model initialization, data order, training budget, base learning-rate schedule and validation protocol. All experiments use the same PyTorch training codebase and NVIDIA A100 80GB GPUs. GPT-2 Small uses one GPU per training trajectory and LLaMA and Qwen3 use four-rank data parallelism.

\subsection{Additional Experimental Results}\label{app-sec:exp-res}
\textbf{Training and Validation Loss for Qwen3 Models.}\quad
We provide the training and validation loss curves for Qwen3-0.6B and Qwen3-1.7B on Pile in Figure~\ref{fig:qwen_pile}.
\begin{figure}[!h]
    \centering

    \subfloat[Qwen3-0.6B on Pile]{
        \begin{minipage}[t]{0.4\linewidth}
            \centering
            \includegraphics[width=\linewidth]{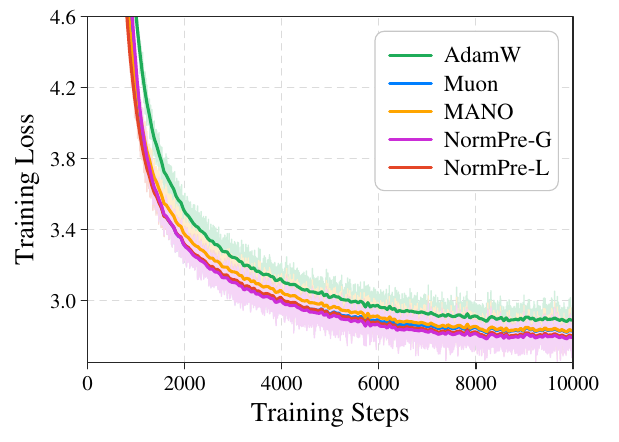}

            \vspace{2pt}

            \includegraphics[width=\linewidth]{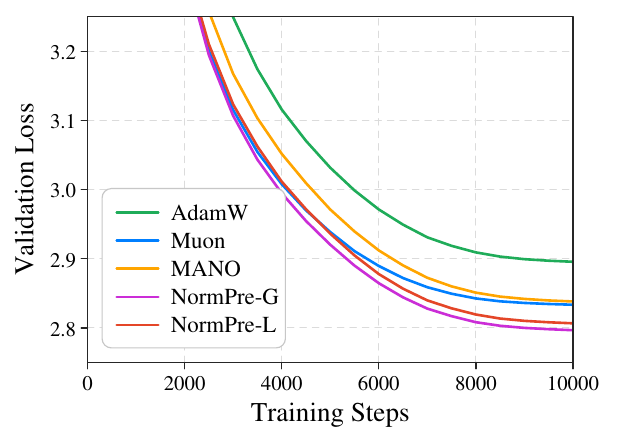}
            \label{fig:qwen3_0p6b_pile_trainval_loss}
        \end{minipage}
    }\hspace{-6pt}%
    \subfloat[Qwen3-1.7B on Pile]{
        \begin{minipage}[t]{0.4\linewidth}
            \centering
            \includegraphics[width=\linewidth]{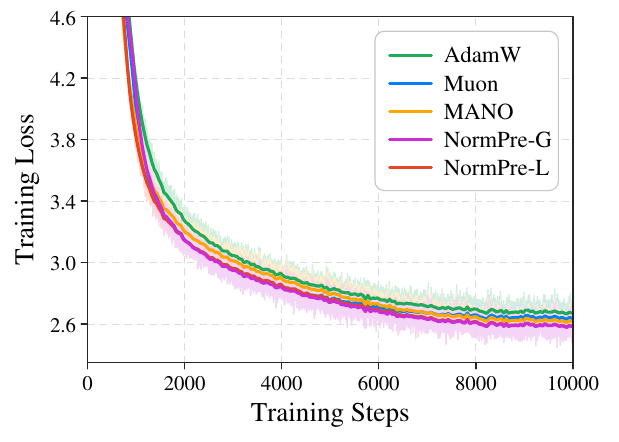}

            \vspace{2pt}

            \includegraphics[width=\linewidth]{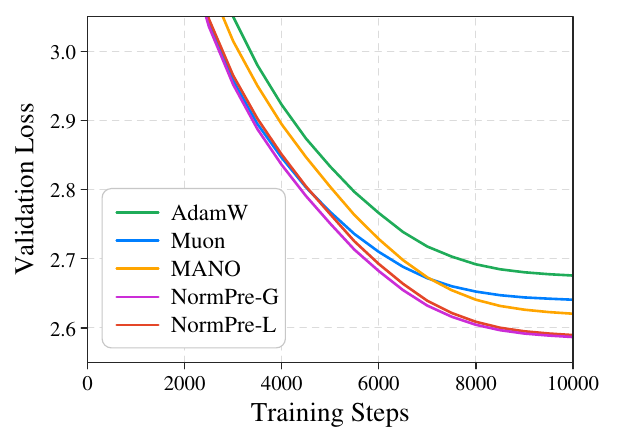}
            \label{fig:qwen3_1p7b_pile_trainval_loss}
        \end{minipage}
    }
    \caption{\textbf{Qwen3 Models on Pile.} Training and validation loss for Qwen3-\{0.6B, 1.7B\} with AdamW, Muon, MANO, \methodG and \methodL. For training loss, dark curves denote 50-step moving averages and light curves the raw trajectories.}
    \label{fig:qwen_pile}
\end{figure}

\textbf{Training Efficiency Configurations.}\quad
Efficiency measurements mirror the main pretraining configurations (including model architecture, sequence length of 1024, effective batch size of 512, and gradient clipping) with the following specific profiling protocols:
\begin{itemize}[leftmargin=*,nosep]
    \item \textit{Precision \& Data Loading:} GPT-2 Small and LLaMA use FP32 parameters with BF16 autocast, while Qwen3 uses pure BF16. OpenWebText is fully materialized in host memory, whereas C4 and Pile are read dynamically from persistent NumPy memmaps (transferring only the active batch to the GPU to avoid full-dataset materialization).
    \item \textit{Measurement Window:} We profile 100 optimizer steps following 20 warmup steps. To isolate core algorithmic overhead, we disable optimizer diagnostics (e.g., descent alignment), validation, checkpointing, W\&B logging, and \texttt{torch.compile}.
    \item \textit{Metric Definitions:} Optimizer latency isolates \texttt{optimizer.step()} using strict CUDA synchronization immediately before and after the call. E2E step time encompasses the complete training loop (data transfer, forward/backward passes, gradient accumulation, clipping, and optimizer update). Throughput is computed as global tokens per E2E step time, and peak memory tracks the maximum CUDA allocated memory during a measured step.
    \item \textit{TF32 Settings:} Unless otherwise stated, TF32 refers to the CUDA matrix-multiplication TF32 setting. To isolate intrinsic algorithmic scaling from hardware-specific Tensor Core accelerations at small dimensions, the main GPT-2 Small efficiency results disable TF32.
\end{itemize}

\textbf{Training Efficiency on Qwen3-1.7B.}\quad
We further evaluate training efficiency on Qwen3-1.7B with Pile under the main experimental configuration. Similar to the observations in Section~\ref{sec:efficiency}, relative to Muon, \methodG increases optimizer latency by $6.77\%$, with an end-to-end step-time overhead of $0.53\%$. \methodL reduces optimizer latency by $57.16\%$, yielding a $4.21\%$ reduction in step time and a $4.41\%$ increase in throughput. Both variants remain slower per step than AdamW and MANO. All evaluated optimizers have the same whole-training peak allocated memory of $53.17$ GiB per GPU. This shared peak is determined by forward and backward computation, whose memory peak exceeds that of the optimizer step.

\begin{table}[t]
    \centering
    \caption{\textbf{Training Efficiency on Qwen3-1.7B.}
    Measurements use the same training configuration as the main Qwen3-1.7B experiment and are averaged over 100 optimizer steps after 20 warmup steps. We report global throughput across 4 GPUs and per-GPU peak allocated memory.}
    \label{tab:qwen-training-efficiency}
    \small
    \renewcommand{\arraystretch}{1.2}
    \setlength{\tabcolsep}{5pt}
    \resizebox{\linewidth}{!}{%
    \begin{tabular}{ccccccc}
        \toprule[1.5pt]
        Model
        & Dataset
        & Optimizer
        & Optimizer Latency $\downarrow$
        & E2E Step Time $\downarrow$
        & Throughput $\uparrow$
        & Peak Memory $\downarrow$ \\
        &
        &
        & (ms/step)
        & (ms/step)
        & (k tokens/s)
        & (GiB) \\
        \midrule
    
        \multirow{5}{*}{Qwen3-1.7B}
        & \multirow{5}{*}{Pile}
        & \textcolor{adamwcolor}{\textbf{AdamW}}
        & 18.3
        & 12102.7
        & 43.32
        & 53.17 \\
    
        &
        & \textcolor{muoncolor}{\textbf{Muon}}
        & 940.0
        & 13050.4
        & 40.17
        & 53.17 \\
    
        &
        & \textcolor{manocolor}{\textbf{MANO}}
        & 118.7
        & 12217.4
        & 42.91
        & 53.17 \\
    
        &
        & \textcolor{normpregcolor}{\textbf{\methodG}}
        & 1003.6\rlap{\;\textcolor{muoncolor}{\scriptsize$\uparrow$6.77\%}}
        & 13119.7\rlap{\;\textcolor{muoncolor}{\scriptsize$\uparrow$0.53\%}}
        & 39.96\rlap{\;\textcolor{muoncolor}{\scriptsize$\downarrow$0.52\%}}
        & 53.17 \\
    
        &
        & \textcolor{normprecolor}{\textbf{\methodL}}
        & 402.7\rlap{\;\textcolor{muoncolor}{\scriptsize$\downarrow$57.16\%}}
        & 12501.4\rlap{\;\textcolor{muoncolor}{\scriptsize$\downarrow$4.21\%}}
        & 41.94\rlap{\;\textcolor{muoncolor}{\scriptsize$\uparrow$4.41\%}}
        & 53.17 \\
        \bottomrule[1.5pt]
    \end{tabular}%
    }
\end{table}

\subsection{Additional Details on Spectral and Training Dynamics}\label{app-sec:dynamics}
\textbf{Spectral Dynamics.}\quad
We collect spectral diagnostics from a NormPre-L training run on GPT-2 Small with OpenWebText and random seed 1337. We use the Sketch implementation with rank $r=32$, oversampling $o=8$, and one power iteration ($p=1$). Diagnostics cover the 48 attention and MLP weight matrices at checkpoints corresponding to steps 100 (101), 500 (501), 1,000 (1,001), 2,000 (2,001), 5,000 (5,001), and 10,000 (9,999). Figure~\ref{fig:before_and_after_spectrum_selected} presents row-active spectra, and Figure~\ref{fig:column_before_and_after_spectrum_selected} provides the column-active counterpart.

\begin{figure}[t]
    \centering
    \subfloat[]{
        \includegraphics[width=0.4\linewidth]{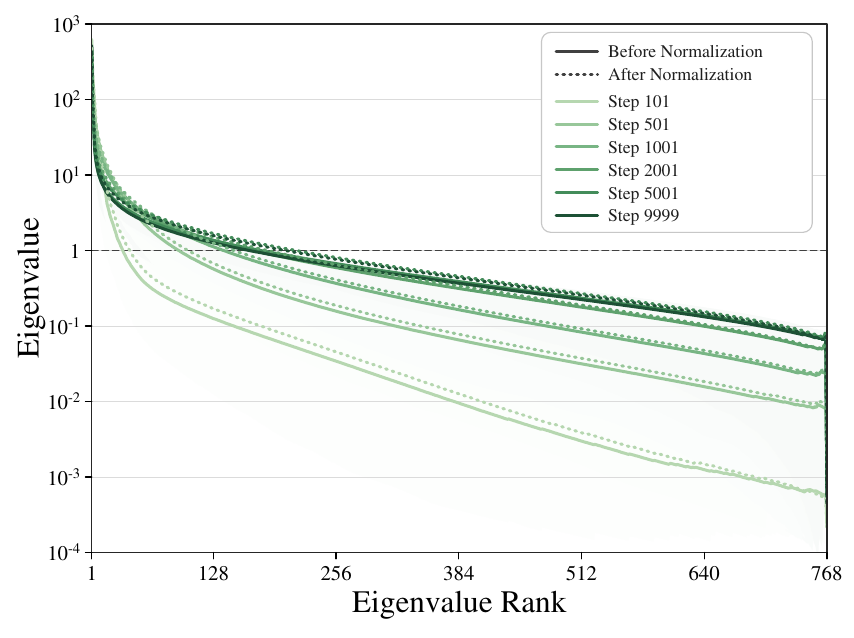}
        \label{fig:column_before_and_after_spectrum_selected}
    }\hspace{12pt}%
    \subfloat[]{
        \includegraphics[width=0.4\linewidth]{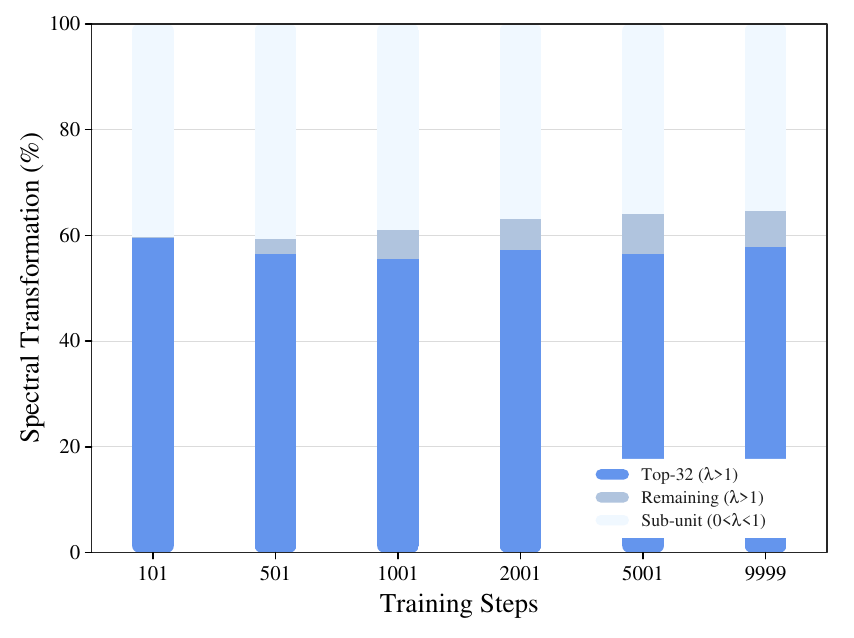}
        \label{fig:column_spectral_transformation}
    }
    \caption{\textbf{Spectral Dynamics.} (a) Eigenspectra before and after marginal normalization in the column-active orientation. (b) Spectral transformation percentages across modes under global and localized preconditioning (column-active orientation).}
    \label{fig:app-spectral}
\end{figure}

\textbf{(1) Eigenspectra Before and After Normalization.}\quad
Following Section~\ref{sec:method}, let $X$ denote the update after relaxed tangent processing (i.e., before normalization) and $\Psi$ denote its marginally normalized counterpart. Both are expressed in the active orientation and measured before spectral preconditioning. To isolate changes in spectral shape from differences in global update magnitude, we compare
$$
\widehat G_X=\frac{d}{\operatorname{tr}(XX^\top)}XX^\top, \quad \Gamma=\Psi\Psi^\top,
$$
where $d$ is the active-axis dimension. Both matrices have trace $d$ and then the same total spectral mass. The scalar rescaling of $XX^\top$ preserves its eigenvalue ratios and relative row-norm differences, which is used only for analysis in spectral dynamics.
We sort the nonzero eigenvalues of $\widehat G_X$ and $\Gamma$ in descending order and average them across matrices at each rank. The shaded regions show the corresponding 25th--75th percentile ranges. Different colors correspond to different training steps. Solid curves show the spectra before marginal normalization ($\widehat G_X$), and dotted curves show those after normalization ($\Gamma$). The horizontal line at $\lambda=1$ indicates the target under full-spectrum preconditioning and localized preconditioning.

\textit{Effect of Marginal Normalization on the Spectrum.}
Before normalization, correlations between row directions are weighted by the products of their row norms in the Gram matrix. Normalization gives each row unit norm, removing these weights and increasing the relative influence of smaller-norm rows. To illustrate this effect, consider mutually orthogonal rows with norms $a_1,\ldots,a_d>0$. The trace-matched Gram matrix is diagonal, with eigenvalues $d a_i^2/\sum_j a_j^2=a_i^2/\overline{a^2}$, where $\overline{a^2}=\frac{1}{d}\sum_j a_j^2$. After marginal normalization, the rows remain orthogonal and each has unit norm, so all eigenvalues equal one. Thus, the eigenvalue corresponding to a row increases if its squared norm is below the average and decreases if it is above the average. For nonorthogonal rows, eigenvalues do not directly correspond to individual row norms. Removing row-norm weighting still changes the spectral structure formed jointly by the rows and might increase eigenvalues across multiple ranks. This provides a possible explanation for the dotted curves lying above the solid curves over many ranks in Figures~\ref{fig:before_and_after_spectrum_selected} and \ref{fig:column_before_and_after_spectrum_selected}.

\textbf{(2) Global and Localized Spectral Preconditioning.}\quad
We use row-active and column-active snapshots from the same NormPre-L trajectory described above. For each snapshot, we compare the spectral transformations induced by global and localized preconditioning on the same normalized interaction spectrum.

Let $\{\lambda_i\}$ denote the nonzero eigenvalues of $\Gamma=\Psi\Psi^\top$. Global preconditioning maps each corresponding singular value $\sqrt{\lambda_i}$ to one. We define the spectral transformation energy of each mode as its squared singular-value change, $e_i=(\sqrt{\lambda_i}-1)^2$, and partition the total into three components:
$$
E_{\mathrm{top}}=\sum_{i\in C}e_i,
\quad
E_{\mathrm{remain}}=\sum_{\substack{\lambda_i>1,i\notin C}}e_i,
\quad
E_{\mathrm{sub}}=\sum_{0<\lambda_i<1}e_i,
$$
where $C$ contains the indices of the largest at most $r=32$ eigenvalues exceeding one. Modes with $\lambda_i=1$ contribute zero. The total spectral transformation energy is $E_{\mathrm{global}}=E_{\mathrm{top}}+E_{\mathrm{remain}}+E_{\mathrm{sub}}$. For each matrix, we express each component as a percentage of $E_{\mathrm{global}}$. The stacked bars show these percentages averaged equally across the 48 matrices at each checkpoint. This decomposition distinguishes the spectral scope of the two variants. Localized preconditioning contracts only the modes in $C$ and leaves the others unchanged. Global preconditioning additionally contracts the remaining modes above one and amplifies the positive modes below one.

Figure~\ref{fig:column_spectral_transformation} presents the column-active results, which follow a similar pattern to the row-active results in Figure~\ref{fig:row_spectral_transformation}. The top-32 modes account for approximately $55\%$--$60\%$ of the total spectral transformation energy across checkpoints.

\begin{figure}[t]
    \centering
    \subfloat[Gradient Norm]{
        \includegraphics[width=0.4\linewidth]{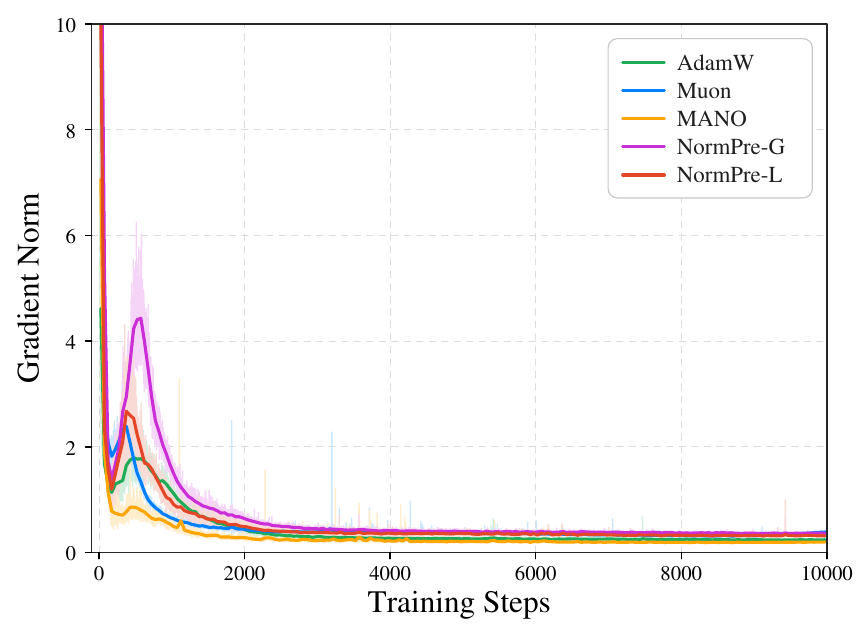}
        \label{fig:gpt2small_gradient_norm_smoothed}
    }
    \hspace{12pt}%
    \subfloat[Descent Alignment]{
        \includegraphics[width=0.4\linewidth]{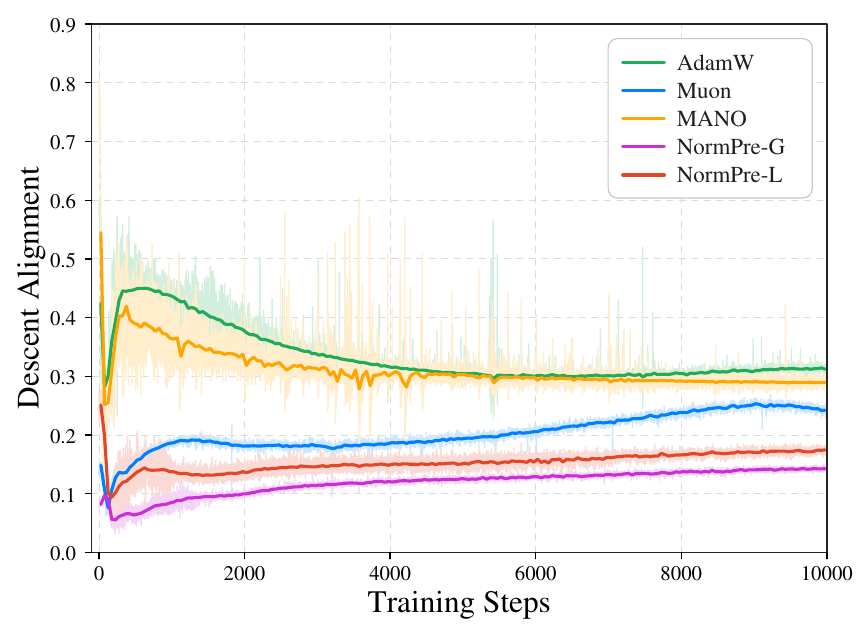}
        \label{fig:gpt2small_descent_cosine_smoothed}
    }
    \caption{\textbf{Training Dynamics.} (a) Pre-clipping global gradient norm over all trainable parameters. (b) Descent alignment between the gradient and optimizer update direction over matrix-optimized parameters. Dark curves show 50-step moving averages and light curves show the corresponding raw trajectories.}
    \label{fig:training_dynamics}
\end{figure}

\textbf{Training Dynamics.}\quad
We examine training dynamics on GPT-2 Small with OpenWebText. Gradient norm is measured over all trainable parameters before clipping. Descent alignment is computed over the attention and MLP weight matrices. For each matrix $W$, let $G_t^{(W)}$ denote its gradient and $\Phi_t^{(W)}$ its optimizer update direction before learning-rate scaling, excluding decoupled weight decay. We report
$$
\operatorname{Align}_t
=\frac{1}{|\mathcal{M}|}
\sum_{W\in\mathcal{M}}
\frac{\langle G_t^{(W)},\Phi_t^{(W)}\rangle_F}
{\|G_t^{(W)}\|_F\|\Phi_t^{(W)}\|_F},
$$
where $\mathcal{M}$ is the set of evaluated weight matrices. This averages the per-matrix cosine similarities. Each cosine equivalently measures the alignment between the optimizer-induced parameter change and the current negative gradient.

\textit{Analysis.} Both NormPre variants exhibit transient peaks in gradient norm during early training, with a more pronounced peak for \methodG. Their descent alignment is also lower than Muon's, indicating greater deviation from the current negative gradient. These distinct early dynamics do not develop into persistent gradient instability. As training progresses, their gradient norms decrease and stabilize, with fewer isolated spikes overall than Muon. Directional differences persist after the gradient norms stabilize: both variants maintain lower alignment than Muon, with \methodG remaining the lowest. Thus, NormPre continues to reshape the update direction throughout training. Together with the lower validation losses, these observations demonstrate an effective update geometry: NormPre adjusts update directions while maintaining positive average descent alignment, sustains relatively steady optimization dynamics in later training, and achieves better final solutions.

\subsection{Additional Ablations}\label{app:additional-ablations}
\begin{table}[t]
    \centering
    \caption{Ablations on Marginally Normalized Update. Results are reported on GPT-2 Small with OpenWebText using seed $1337$ under the same $10$k-step training budget.}
    \label{tab:ablation-main}
    \small
    \renewcommand{\arraystretch}{1.2}
    \setlength{\tabcolsep}{6pt}
    
    \begin{minipage}[t]{0.43\linewidth}
        \centering
        (a) Momentum Processing\\[3pt]
        \begin{tabular}{@{}lcc@{}}
            \toprule[1.5pt]
            \multirow{2}{*}{Variant} & \multicolumn{2}{c}{Validation Loss} \\
            \cmidrule(lr){2-3}
            & \methodG & \methodL \\
            \midrule
            No Tangent      & 3.0945 & 3.1139 \\
            Strict Tangent  & 3.1869 & 3.1196 \\
            Relaxed Tangent & 3.0711 & 3.0864 \\
            \bottomrule[1.5pt]
        \end{tabular}
    \end{minipage}
    \hfill
    \begin{minipage}[t]{0.52\linewidth}
        \centering
        (b) Normalization Orientation\\[3pt]
        \begin{tabular}{@{}lcc@{}}
            \toprule[1.5pt]
            \multirow{2}{*}{Orientation} & \multicolumn{2}{c}{Validation Loss} \\
            \cmidrule(lr){2-3}
            & \methodG & \methodL \\
            \midrule
            Row                    & 3.0604 & 3.0817 \\
            Alternating Row/Column & 3.0711 & 3.0864 \\
            Row+Column/Step        & 3.0735 & 3.0878 \\
            \bottomrule[1.5pt]
        \end{tabular}
    \end{minipage}
\end{table}

\begin{table}[t]
    \centering
    \caption{Ablations on Interaction Geometry Source.
    Results on GPT-2 Small with OpenWebText using seed 1337 and the same 10k-step training budget.}
    \label{tab:interaction-source-ablation}
    \small
    \renewcommand{\arraystretch}{1.15}
    \setlength{\tabcolsep}{8pt}

    \resizebox{\linewidth}{!}{%
    \begin{tabular}{lcccc}
        \toprule[1.5pt]
        \multirow{2}{*}{Interaction Source}
        & \multicolumn{2}{c}{\methodG}
        & \multicolumn{2}{c}{\methodL} \\
        \cmidrule(lr){2-3}\cmidrule(lr){4-5}
        & Training Loss
        & Validation Loss
        & Training Loss
        & Validation Loss \\
        \midrule

        Update before normalization ($X$)
        & 3.0844
        & 3.0900
        & 3.1375
        & 3.1305 \\

        Update after normalization ($\Psi$)
        & \textbf{3.0655}
        & \textbf{3.0711}
        & \textbf{3.0810}
        & \textbf{3.0864} \\

        \bottomrule[1.5pt]
    \end{tabular}
    }
\end{table}

\begin{table}[t]
    \centering
    \caption{Spectral Budget in \methodL. Results are reported on GPT-2 Small with OpenWebText using seed $1337$ under the same $10$k-step training budget.}
    \label{tab:spectral-budget}
    \small
    \renewcommand{\arraystretch}{1.2}
    \begin{tabular}{lcc}
        \toprule[1.5pt]
        Rank $r$ & Training Loss & Validation Loss \\
        \midrule
        8            & 3.0911 & 3.0962 \\
        16           & 3.0857 & 3.0908 \\
        32 (default) & 3.0810 & 3.0864 \\
        64           & 3.0768 & 3.0827 \\
        \bottomrule[1.5pt]
    \end{tabular}
\end{table}

\textbf{Ablations on Marginally Normalized Update.}\quad
Table~\ref{tab:ablation-main} examines two design choices in constructing the marginally normalized update for both \methodG and \methodL. In part (a), relaxed tangent projection achieves the lowest validation loss for both variants, outperforming both leaving the momentum unchanged and strict tangent projection. The degradation under strict tangent projection is particularly pronounced for \methodG, suggesting that completely removing the radial component can discard useful information before spectral refinement. These results support the relaxed design, which prioritizes direction-changing information while retaining a limited scale-dependent component. From part (b), the choice of normalization orientation has a smaller effect. For \methodL, row-only and alternating row/column normalization achieve comparable validation losses, with a slight advantage for the row-only variant in this controlled comparison. For \methodG, the row-only variant shows a clearer improvement over alternating normalization, while normalizing both axes at every step performs similarly to the alternating scheme. We adopt alternating normalization as the default to treat the two matrix axes symmetrically and avoid introducing a fixed preferred orientation.

\textbf{Ablations on Interaction Geometry Source.}\quad
We compare interaction geometry extracted from the update before normalization ($X$) with that extracted from the update after normalization ($\Psi$). Both settings apply preconditioning to the same normalized base update $\Psi$ and differ only in the source used to construct the preconditioner.

Table~\ref{tab:interaction-source-ablation} shows that extracting interaction geometry after normalization yields lower training and validation loss for both variants. This supports constructing the preconditioner from the geometry of the update on which it acts. Using the update before normalization causes a larger degradation for \methodL, suggesting greater sensitivity to the source when refinement is restricted to selected modes. The marginal scales in $X$ can affect the spectral ordering, so its leading modes may differ from those of $\Psi$. Normalization removes this scale weighting first, allowing the top-$r$ eigenspace to capture interaction structure without the influence of unequal row norms. This offers a possible explanation for the greater sensitivity of localized refinement. Together, these results support using marginal normalization both to construct the base update and to establish the interaction geometry for subsequent spectral preconditioning.

\textbf{Spectral Budget in \methodL.}\quad
Table~\ref{tab:spectral-budget} reports training and validation losses for different spectral budgets, complementing Figure~\ref{fig:gpt2small_rank_localization_frontier_new}. Both losses decrease as $r$ increases over the evaluated range. We use $r=32$ as the default to balance these performance gains against the optimizer cost.

%% file: Appendix/appendix_implementation_complexity.tex
\newpage
\section{Detailed Implementations on Spectral Preconditioning}
\label{app:spectral-implementation}

From Section~\ref{sec:spectral-implementation}, \methodG realizes global spectral preconditioning by applying Newton-Schulz iteration to the normalized update $\Psi$, approximating $T^{\text{G}}=\operatorname{msign}(\Psi)$. This follows the standard Newton-Schulz implementation used by Muon~\citep{jordan2024muon, liu2025muon}. \methodL operates on a selected interaction eigenspace, which can be obtained through exact eigendecomposition or a randomized sketch.
In the following, we focus on the implementation details of \methodL.

\textbf{Exact Eigenspace Extraction.}\quad The exact implementation explicitly forms the interaction matrix $\Gamma=\Psi\Psi^\top$ and computes its eigendecomposition $\Gamma=\widetilde U\Lambda\widetilde U^\top$. Following the localized construction in Section~\ref{sec:prelim-spectral-preconditioning}, we define the active set $\mathcal A:=\{i:\lambda_i>1\}$ and let $\mathcal C\subseteq\mathcal A$ index the $\min\{r,|\mathcal A|\}$ largest active eigenvalues. The selected eigenpairs $(\widetilde U_{\mathcal C},\Lambda_{\mathcal C})$ are directly used in the localized preconditioner
$$
P^\text{L}=I_m+\widetilde U_{\mathcal C}(\Lambda_{\mathcal C}^{-1/2}-I)\widetilde U_{\mathcal C}^\top.
$$

\textbf{Sketch-Based Eigenspace Extraction.}\quad To avoid forming and decomposing the full interaction matrix, we approximate its leading eigenspace through a randomized sketch~\citep{halko2011finding}. Among the top-$r$ estimated interaction modes, those satisfying $\widehat\lambda_i>1$ are retained for the localized preconditioner
$$
\widehat P^\text{L}=I_m+\widehat U_{\mathcal C}(\widehat\Lambda_{\mathcal C}^{-1/2}-I)\widehat U_{\mathcal C}^\top.
$$
Algorithm~\ref{alg:interaction-sketch} gives the complete procedure.

\begin{algorithm}[!h]
    \caption{Randomized Interaction Sketch}
    \label{alg:interaction-sketch}
    \begin{algorithmic}[1]
    
    \Require Marginally normalized update $\Psi\in\mathbb{R}^{m\times n}$, rank $r\le m$, oversampling $o$, power iterations $p$.
    
    \State $\ell\gets\min\{m,r+o\}$
    
    \State Sample $\Omega\in\mathbb{R}^{n\times\ell}$ with $\Omega_{ij}\sim\mathcal N(0,1)$ and set $Y\gets\Psi\Omega$
    \Comment{Gaussian Sketch}
    
    \For{$j=1,\ldots,p$}
        \State $Y\gets\Psi(\Psi^\top Y)$
    \EndFor
    \Comment{Power Iteration}
    
    \State $Q\gets\operatorname{qr}(Y)$
    \Comment{Candidate Subspace}
    
    \State $\Gamma_Q\gets(Q^\top\Psi)(Q^\top\Psi)^\top$
    \Comment{Rayleigh--Ritz Matrix}
    
    \State $(Z,\widehat\Lambda)\gets$ top-$r$ eigenpairs of $\Gamma_Q$
    \Comment{Ritz Pairs}
    
    \State $\widehat U\gets QZ$
    \Comment{Lifted Ritz Vectors}

    \State $\mathcal C\gets\{i:\widehat\lambda_i>1\}$
    
    \State \Return $(\widehat U_{\mathcal C},\widehat\Lambda_{\mathcal C})$
    
    \end{algorithmic}
\end{algorithm}

The main steps of Algorithm~\ref{alg:interaction-sketch} are as follows:
\begin{itemize}[itemsep=0.3em, parsep=0em, topsep=0.3em, partopsep=0em, leftmargin=2em]
    \item \textit{Randomized Subspace Construction (Lines $2$--$6$).} The Gaussian sketch $Y=\Psi\Omega$ probes the range of $\Psi$ and the power iterations $Y\leftarrow\Psi(\Psi^\top Y)$ amplify the separation of its leading singular directions. Then, QR decomposition yields an orthonormal basis $Q\in\mathbb{R}^{m\times\ell}$ for the candidate interaction subspace.

    \item \textit{Rayleigh--Ritz Eigenspace Extraction (Lines $7$--$10$).} The interaction matrix is projected onto the candidate subspace through
    $$
    \Gamma_Q=Q^\top\Gamma Q=(Q^\top\Psi)(Q^\top\Psi)^\top\in\mathbb{R}^{\ell\times\ell}.
    $$
    Computing the top-$r$ eigenpairs $(Z,\widehat\Lambda)$ of $\Gamma_Q$ gives the corresponding Ritz pairs and the Ritz vectors are lifted to the original interaction space through $\widehat U=QZ$. We retain the estimated modes satisfying $\widehat\lambda_i>1$, yielding $(\widehat U_{\mathcal C},\widehat\Lambda_{\mathcal C})$ for localized spectral preconditioning. Since the Ritz values are ordered in descending order, applying $\widehat{\lambda}_i > 1$ after top-$r$ selection is equivalent to first identifying the active modes and then retaining the largest $r$, consistent with Section~\ref{sec:prelim-spectral-preconditioning}.
\end{itemize}

\newpage
\section{Detailed Computational Complexity Analysis}
\label{app-sec:compute_analysis}

This section analyzes the optimizer-side computational complexity. 
For \method, the analysis covers momentum update, relaxed tangent projection, diagonal-Gram normalization, spectral preconditioning, update RMS scaling and parameter update. We consider both \methodG and \methodL, and include Muon as a reference for the Newton-Schulz computation. Table~\ref{tab:optimizer-complexity-summary} summarizes the resulting complexities.

Let $m\times n$ denote the matrix dimensions entering the spectral preconditioning stage and define $s:=\min\{m,n\}$. For the Newton-Schulz implementations in Muon and \methodG, the input is transposed when necessary so that the Gram matrix is $s\times s$. For \methodL, $m\times n$ follows the current active orientation. Let $q$ denote the number of Newton-Schulz iterations, $r$ the interaction rank, $o$ the oversampling parameter, $p$ the number of power iterations and $\ell:=\min\{m,r+o\}$. We omit constant factors and lower-order terms.

\begin{table}[!ht]
    \centering
    \caption{Computational Complexity of Muon and NormPre Variants.}
    \label{tab:optimizer-complexity-summary}
    \small
    \begin{tabular}{lc}
        \toprule
        \textbf{Method} & \textbf{Computational Complexity} \\
        \midrule
        Muon
        & $\mathcal O(mn+qmns)$ \\
        NormPre-G
        & $\mathcal O(mn+qmns)$ \\
        NormPre-L (Exact)
        & $\mathcal O(m^2n+m^3)$ \\
        NormPre-L (Sketch)
        & $\mathcal O\!\left((p+1)mn\ell+(m+n)\ell^2+\ell^3\right)$ \\
        \bottomrule
    \end{tabular}
\end{table}

\subsection{Shared Components in NormPre-G and NormPre-L}\label{app:complexity-shared}
Following Algorithm~\ref{alg:normpre}, NormPre-G and NormPre-L share the same optimization pipeline except for the spectral preconditioning step. We analyze these shared operations under the active orientation.

\begin{itemize}[itemsep=0.4em, parsep=0em, topsep=0.3em, partopsep=0em, leftmargin=2em]
    \item \textit{Momentum Update and Relaxed Tangent Projection.}
    Updating the momentum requires $\mathcal O(mn)$ operations. The tangent projection is applied independently to each row with $\mathcal O(n)$ cost per row (inner products and vector rescaling), giving $\mathcal O(mn)$ operations in total.
    
    \item \textit{Diagonal-Gram Normalization.}
    Computing the row norms of $X$ and forming the normalized update $\Psi=D(X)^{-1}X$ require $\mathcal O(mn)$ operations.

    \item \textit{Update RMS Scaling and Parameter Update.}
    Computing $\operatorname{RMS}(T)$ and rescaling the update both require $\mathcal O(mn)$ operations. Applying weight decay and updating the parameter matrix also require $\mathcal O(mn)$ operations.
\end{itemize}

Overall, these shared operations contribute $\mathcal O(mn)$ computation to both variants and are included in the total costs reported in Table~\ref{tab:optimizer-complexity-summary}.

\subsection{NormPre-G and Muon}
\textbf{Muon.}\quad
Muon applies a Newton-Schulz approximation of the matrix-sign transformation to its momentum update. The input is transposed when necessary so that the Gram matrix is $s\times s$. Each Newton-Schulz iteration involves matrix products between the $s\times s$ Gram matrix and the matrix update, with dominant cost $\mathcal O(mns)$. For $q$ iterations, the matrix transformation requires $\mathcal O(qmns)$. The remaining optimizer operations (\textit{e.g.}, momentum update, update RMS scaling and parameter update) contribute $\mathcal O(mn)$, giving the total complexity $\mathcal O(mn+qmns).$

\textbf{NormPre-G.}\quad
NormPre-G applies the same $q$-step Newton-Schulz transformation to the normalized update $\Psi$. Thus, its spectral preconditioning stage requires $\mathcal O(qmns)$ computation. Together with the shared $\mathcal O(mn)$ operations analyzed in Section~\ref{app:complexity-shared}, the total complexity is $\mathcal O(mn+qmns).$
Therefore, Muon and NormPre-G have the same leading asymptotic computational complexity for a fixed number of Newton-Schulz iterations. The relaxed tangent projection and diagonal-Gram normalization in NormPre-G add only $\mathcal O(mn)$ computation and do not change the asymptotic order.

\subsection{NormPre-L}
For the Exact implementation, the selected eigenvectors $\widetilde U_{\mathcal C}\in\mathbb R^{m\times|\mathcal C|}$ and eigenvalues $\Lambda_{\mathcal C}$ give $P^{\mathrm{L}}\Psi = \Psi+\widetilde U_{\mathcal C}\left(\Lambda_{\mathcal C}^{-1/2}-I\right)\widetilde U_{\mathcal C}^{\top}\Psi.$ For the Sketch-based implementation, the estimated eigenpairs $(\widehat U_{\mathcal C},\widehat\Lambda_{\mathcal C})$ yield $\widehat P^{\mathrm{L}}\Psi=\Psi+\widehat U_{\mathcal C}\left(\widehat\Lambda_{\mathcal C}^{-1/2}-I\right)\widehat U_{\mathcal C}^{\top}\Psi.$
In both cases, the selected eigenvector matrix has dimension $m\times|\mathcal C|$ with $|\mathcal C|\le r$. Computing $\widetilde U_{\mathcal C}^{\top}\Psi$ or $\widehat U_{\mathcal C}^{\top}\Psi$ requires $\mathcal O(mn|\mathcal C|)$ operations. Applying the diagonal spectral scaling costs $\mathcal O(n|\mathcal C|)$ and multiplying by $\widetilde U_{\mathcal C}$ or $\widehat U_{\mathcal C}$ requires another $\mathcal O(mn|\mathcal C|)$ operations. 
Therefore, the localized spectral application costs $\mathcal O(mnr)$. The exact and sketch-based implementations differ in the cost of extracting the selected interaction eigenspace.

\textbf{NormPre-L (Exact).}\quad
The exact implementation forms the interaction matrix $\Gamma=\Psi\Psi^\top$, which costs $\mathcal O(m^2n)$. The resulting $m\times m$ matrix is then processed by a dense symmetric eigendecomposition with cost $\mathcal O(m^3)$. 
Together with the shared $\mathcal O(mn)$ operations analyzed in Section~\ref{app:complexity-shared} and the $\mathcal O(mnr)$ localized spectral application, we obtain $\mathcal O\!\left(mn+mnr+m^2n+m^3\right)=\mathcal O(m^2n+m^3)$, where the simplification follows from $r\leq m$.

\textbf{NormPre-L (Sketch).}\quad
The sketch-based implementation avoids forming and decomposing the full $m\times m$ interaction matrix by using the randomized sketch in Algorithm~\ref{alg:interaction-sketch}. Its eigenspace-extraction cost can be decomposed as follows, the main intermediate dimensions are summarized in Table~\ref{tab:sketch-dimensions}.

\begin{itemize}[itemsep=0.5em, parsep=0em, topsep=0.3em, partopsep=0em, leftmargin=2em]
    \item \textit{Randomized Sketch and Power Iteration.}
    Forming $Y=\Psi\Omega$ costs $\mathcal O(mn\ell)$. 
    Each power iteration computes $\Psi^\top Y$ followed by $\Psi(\Psi^\top Y)$, with both matrix multiplications requiring $\mathcal O(mn\ell)$ operations. Over $p$ iterations, the power-iteration cost is $\mathcal O(pmn\ell)$.

    \item \textit{Candidate Subspace.}
    Computing the QR decomposition of $Y\in\mathbb{R}^{m\times\ell}$ costs $\mathcal O(m\ell^2)$.

    \item \textit{Rayleigh--Ritz Matrix.}
    Forming $Q^\top\Psi\in\mathbb{R}^{\ell\times n}$ costs $\mathcal O(mn\ell)$ and constructing $\Gamma_Q=(Q^\top\Psi)(Q^\top\Psi)^\top\in\mathbb{R}^{\ell\times\ell}$ costs $\mathcal O(n\ell^2)$.

    \item \textit{Rayleigh--Ritz Extraction.}
    A dense eigendecomposition of $\Gamma_Q\in\mathbb{R}^{\ell\times\ell}$ costs $\mathcal O(\ell^3)$. Computing $\widehat U=QZ$ costs $\mathcal O(m\ell r)$, which is bounded by $\mathcal O(m\ell^2)$ because $r\leq\ell$.
\end{itemize}

We have $\mathcal O\!\left(mn+mnr+(p+1)mn\ell+(m+n)\ell^2+\ell^3\right)$ when combining these terms with the shared $\mathcal O(mn)$ operations and $\mathcal O(mnr)$ localized spectral application. Since $r\leq\ell$, the $\mathcal O(mn)$ and $\mathcal O(mnr)$ terms are absorbed into the sketching cost. Therefore, the total computational complexity of NormPre-L (Sketch) is $\mathcal O\!\left((p+1)mn\ell+(m+n)\ell^2+\ell^3\right).$

\begin{table}[!h]
    \centering
    \caption{Dimensions of the Main Quantities in Algorithm~\ref{alg:interaction-sketch}}
    \label{tab:sketch-dimensions}
    \small
    \begin{tabular}{lll}
        \toprule
        \textbf{Symbol} & \textbf{Dimension} & \textbf{Description} \\
        \midrule
        $\Psi$ & $m\times n$ & Marginally normalized update \\
        $\Omega$ & $n\times\ell$ & Gaussian random matrix \\
        $Y$ & $m\times\ell$ & Gaussian sketch \\
        $Q$ & $m\times\ell$ & Candidate subspace \\
        $\Gamma_Q$ & $\ell\times\ell$ & Rayleigh--Ritz matrix \\
        $Z$ & $\ell\times r$ & Top-$r$ eigenvectors of $\Gamma_Q$ \\
        $\widehat\Lambda$ & $r\times r$ & Estimated eigenvalues \\
        $\widehat U$ & $m\times r$ & Estimated eigenvectors \\
        \bottomrule
    \end{tabular}
\end{table}

%% file: Appendix/appendix_prop.tex
\newpage
\section{Proofs}\label{app:proof}
\subsection{Optimization Formulation of Localized Spectral Refinement}
\label{app:proof-regularized-spectral-steepest-descent}

\begin{proposition}[Regularized Spectral Steepest Descent]
    \label{prop:regularized_spectral_sd}
    Suppose that $\Gamma(X)\succ0$. The closest update to $\Psi$ within the unit spectral-norm ball is
    \begin{align}
        T_+:=\arg\min_{\|T\|_{\mathrm{op}}\leq1}\frac{1}{2}\|T-\Psi\|_F^2=\arg\max_{\|T\|_{\mathrm{op}}\leq1}\left\{\langle\Psi,T\rangle_F-\frac{1}{2}\|T\|_F^2\right\}.
        \label{eq:regularized-steepest-problem}
    \end{align}
    Its unique solution is $T_+=P_+\Psi$, where the preconditioner is
    \begin{align}
        P_+=\widetilde U\operatorname{Diag}\left(\min\{1,\lambda_1^{-1/2}\},\ldots,\min\{1,\lambda_m^{-1/2}\}\right)\widetilde U^\top.
        \label{eq:regularized-steepest-preconditioner}
    \end{align}
    In particular, $0\prec P_+\preceq I_m$ and $\sigma_i(T_+)=\min\{\widetilde\sigma_i,1\}$.
\end{proposition}
\begin{remark}
    Proposition~\ref{prop:regularized_spectral_sd} gives the exact projection of $\Psi$ onto the unit spectral-norm ball, clipping singular values above one to one while leaving the remaining singular values unchanged.
    Define the active set $\mathcal A:=\{i:\lambda_i>1\}$, together with $\widetilde U_{\mathcal A}:=[\widetilde u_i]_{i\in\mathcal A}$ and $\Lambda_{\mathcal A}:=\operatorname{Diag}(\lambda_i)_{i\in\mathcal A}$. Equivalently, the preconditioner can be written as 
    $$
    P_+=I_m+\widetilde U_{\mathcal A}(\Lambda_{\mathcal A}^{-1/2}-I)\widetilde U_{\mathcal A}^\top.
    $$
    Its non-identity component is $\Delta_{\mathcal A}:=P_+-I_m$. Hence, $P_+\Psi=\Psi+\Delta_{\mathcal A}\Psi$, where $\Delta_{\mathcal A}\Psi$ is the induced update refinement relative to the normalized update $\Psi$. 
\end{remark}

\begin{proof}
    Let $\Psi=\widetilde U\widetilde\Sigma\widetilde V^\top$ be a thin SVD, where $\widetilde\Sigma=\operatorname{Diag}(\widetilde\sigma_1,\ldots,\widetilde\sigma_m)$ with $\widetilde\sigma_1\geq\cdots\geq\widetilde\sigma_m>0$. Suppose that $\Gamma(X)=\Psi\Psi^\top=\widetilde U\Lambda \widetilde U^\top \succ0$, i.e., $\Psi$ has full row rank and $\lambda_i=\widetilde\sigma_i^2>0$ for every $i$.

    For the regularized steepest-descent problem,
    $$
    \frac{1}{2}\|T-\Psi\|_F^2=\frac{1}{2}\|\Psi\|_F^2-\left(\langle\Psi,T\rangle_F-\frac{1}{2}\|T\|_F^2\right).
    $$
    Since $\|\Psi\|_F^2/2$ is independent of $T$, the two optimization problems in Equation~\ref{eq:regularized-steepest-problem} are equivalent.

    Following the standard spectral steepest-descent derivation of the matrix-sign update, we solve the regularized problem under the same unit spectral-norm constraint by optimizing its singular values modewise.
    
    Let $\tau_1,\ldots,\tau_m$ denote the singular values of a feasible $T$. The constraint $\|T\|_{\mathrm{op}}\leq1$ implies $0\leq\tau_i\leq1$. By von Neumann's trace inequality, $\langle\Psi,T\rangle_F\leq\sum_{i=1}^m\widetilde\sigma_i\tau_i,$ with equality when $T$ shares the singular directions of $\Psi$. Since $\|T\|_F^2=\sum_{i=1}^m\tau_i^2$, the regularized objective satisfies
    $$
    \langle\Psi,T\rangle_F-\frac{1}{2}\|T\|_F^2\leq\sum_{i=1}^m\left(\widetilde\sigma_i\tau_i-\frac{1}{2}\tau_i^2\right).
    $$
    Thus, the matrix problem reduces to independently maximizing $\widetilde\sigma_i\tau_i-\frac{1}{2}\tau_i^2$ over $0\leq\tau_i\leq1$ for each singular mode.
    For each $i$, consider the scalar objective $f_i(\tau_i):=\widetilde\sigma_i\tau_i-\frac{1}{2}\tau_i^2.$ Since $f_i'(\tau_i)=\widetilde\sigma_i-\tau_i$ and $f_i''(\tau_i)=-1<0$, its unique unconstrained maximizer is $\tau_i=\widetilde\sigma_i$. With the unit spectral-norm constraint $0\leq\tau_i\leq1$, we have 
    $$
    \tau_i^\star=\min\{\widetilde\sigma_i,1\}.
    $$
    Therefore, the unique solution is
    $$
    T_+=\widetilde U\operatorname{Diag}\left(\min\{\widetilde\sigma_1,1\},\ldots,\min\{\widetilde\sigma_m,1\}\right)\widetilde V^\top.
    $$
    The uniqueness follows from the strict concavity of $\langle\Psi,T\rangle_F-\frac{1}{2}\|T\|_F^2$ over the convex feasible set.

    We next express the optimal solution as a left spectral preconditioning of $\Psi$. Since $\lambda_i=\widetilde\sigma_i^2$, define
    $$
    P_+:=\widetilde U\operatorname{Diag}\left(\min\{1,\lambda_1^{-1/2}\},\ldots,\min\{1,\lambda_m^{-1/2}\}\right)\widetilde U^\top.
    $$
    Then $P_+\Psi
    =\widetilde U\operatorname{Diag}\left(\min\{1,\widetilde\sigma_1^{-1}\}\widetilde\sigma_1,\ldots,\min\{1,\widetilde\sigma_m^{-1}\}\widetilde\sigma_m\right)\widetilde V^\top
    =T_+$, i.e., $T_+=P_+\Psi$.
    Since
    $$
    \min\{1,\lambda_i^{-1/2}\}
    =
    \begin{cases}
    1, & \lambda_i\leq1,\\
    \lambda_i^{-1/2}, & \lambda_i>1,
    \end{cases}
    $$
    letting $\mathcal A:=\{i:\lambda_i>1\}$ gives the equivalent representation of the preconditioner
    $$
    P_+=I_m+\widetilde U_{\mathcal A}\left(\Lambda_{\mathcal A}^{-1/2}-I\right)\widetilde U_{\mathcal A}^\top.
    $$
    Thus, $P_+$ leaves modes with $\lambda_i\leq1$ unchanged and contracts modes with $\lambda_i>1$ by $\lambda_i^{-1/2}$. Moreover, every eigenvalue of $P_+$ lies in $(0,1]$, so $0\prec P_+\preceq I_m$.
\end{proof}

\begin{proposition}[Rank-Constrained Spectral Refinement]
    \label{prop:optimal-rank-constrained-refinement}
    Suppose that $\Gamma(X)\succ0$. A closest rank-constrained update refinement to $\Delta_{\mathcal A}\Psi$ is obtained by solving
    \begin{align}
        \Delta^\star
        \in
        \arg\min_{\operatorname{rank}(\Delta)\leq r}
        \frac{1}{2}\left\|(\Delta-\Delta_{\mathcal A})\Psi\right\|_F^2.
        \label{eq:optimal-rank-constrained-refinement}
    \end{align}
    Let $\mathcal C\subseteq\mathcal A$ index the $\min\{r,|\mathcal A|\}$ largest eigenvalues satisfying $\lambda_i>1$. Then one solution is
    \begin{align}
        \Delta^\star
        =\widetilde U_{\mathcal C}
        \left(\Lambda_{\mathcal C}^{-1/2}-I\right)
        \widetilde U_{\mathcal C}^\top.
        \label{eq:optimal-interaction-subspace}
    \end{align}
    The resulting preconditioner and update are $P:=I_m+\Delta^\star$ and $T:=P\Psi$.
\end{proposition}

\begin{remark}
    Proposition~\ref{prop:optimal-rank-constrained-refinement} characterizes the optimal rank-constrained spectral refinement. Since $\Gamma(X)\succ0$, $\Psi$ has full row rank and $\operatorname{rank}(\Delta\Psi)=\operatorname{rank}(\Delta)$, so the rank constraint on the non-identity component is equivalent to that on its induced update refinement. The result then follows from the Eckart--Young--Mirsky theorem~\citep{eckart1936approximation,mirsky1960symmetric} applied to $\Delta_{\mathcal A}\Psi$. Its nonzero singular values are $|1-\sqrt{\lambda_i}|$ for $i\in\mathcal A$ and increase monotonically with $\lambda_i$, so the optimal truncation retains the active modes associated with the largest $\lambda_i$ indexed by $\mathcal C$.
\end{remark}

\begin{proof}
    Recall that $\Psi=\widetilde U\widetilde\Sigma\widetilde V^\top$ is a thin SVD, where $\widetilde\Sigma=\operatorname{Diag}(\widetilde\sigma_1,\ldots,\widetilde\sigma_m)$ with $\widetilde\sigma_1\geq\cdots\geq\widetilde\sigma_m>0$. Suppose that $\Gamma(X)=\Psi\Psi^\top=\widetilde U\Lambda \widetilde U^\top \succ0$, i.e., $\Psi$ has full row rank and $\lambda_i=\widetilde\sigma_i^2>0$ for every $i$.
    From Proposition~\ref{prop:regularized_spectral_sd}, with $\mathcal A:=\{i:\lambda_i>1\}$ and $P_+=I_m+\widetilde U_{\mathcal A}\left(\Lambda_{\mathcal A}^{-1/2}-I\right)\widetilde U_{\mathcal A}^\top$,
    $$
    \Delta_{\mathcal A}=P_+-I_m=\widetilde U_{\mathcal A}\left(\Lambda_{\mathcal A}^{-1/2}-I\right)\widetilde U_{\mathcal A}^\top.
    $$
    Applying this refinement to $\Psi$, and using that $\widetilde U_{\mathcal A}$ and $\widetilde V_{\mathcal A}$ collect the singular directions indexed by $\mathcal A$, we have
    \begin{align*}
        \Delta_{\mathcal A}\Psi
        &=\widetilde U_{\mathcal A}\left(\Lambda_{\mathcal A}^{-1/2}-I\right)\widetilde U_{\mathcal A}^\top\widetilde U\widetilde\Sigma\widetilde V^\top\\
        &=\widetilde U_{\mathcal A}\left(\Lambda_{\mathcal A}^{-1/2}-I\right)\widetilde\Sigma_{\mathcal A}\widetilde V_{\mathcal A}^\top\\
        &=\widetilde U_{\mathcal A}\operatorname{Diag}\left(1-\widetilde\sigma_i\right)_{i\in\mathcal A}\widetilde V_{\mathcal A}^\top,
    \end{align*}
    where the last equality follows from $\lambda_i=\widetilde\sigma_i^2$.
    Since $\lambda_i>1$ for $i\in\mathcal A$, we have $\widetilde\sigma_i=\sqrt{\lambda_i}>1$. Therefore, the nonzero singular values of $\Delta_{\mathcal A}\Psi$ are $\left|1-\widetilde\sigma_i\right|=\sqrt{\lambda_i}-1$ for $i\in\mathcal A$.

    For any feasible $\Delta$ with $\operatorname{rank}(\Delta)\leq r$, we have $\operatorname{rank}(\Delta\Psi)\leq\operatorname{rank}(\Delta)\leq r.$ Replacing $\Delta\Psi$ by an arbitrary matrix $Z$ with rank at most $r$ enlarges the feasible set. Therefore,
    $$
    \min_{\operatorname{rank}(Z)\leq r}\|Z-\Delta_{\mathcal A}\Psi\|_F^2
    \leq
    \min_{\operatorname{rank}(\Delta)\leq r}\|\Delta\Psi-\Delta_{\mathcal A}\Psi\|_F^2.
    $$
    By the Eckart--Young--Mirsky theorem, the left-hand side is minimized by retaining the $\min\{r,|\mathcal A|\}$ largest singular values of $\Delta_{\mathcal A}\Psi$, whose rank is $|\mathcal A|$.
    Since $\sqrt{\lambda_i}-1$ is increasing in $\lambda_i$, the optimal rank-$r$ approximation retains the modes with the largest active eigenvalues. Let $\mathcal C\subseteq\mathcal A$ index these modes. The resulting truncated approximation is
    $$
    Z^\star=\widetilde U_{\mathcal C}\operatorname{Diag}\left(1-\widetilde\sigma_i\right)_{i\in\mathcal C}\widetilde V_{\mathcal C}^\top.
    $$
    Define $\Delta^\star:=\widetilde U_{\mathcal C}\left(\Lambda_{\mathcal C}^{-1/2}-I\right)\widetilde U_{\mathcal C}^\top.$
    Since $|\mathcal C|\leq r$, we have $\operatorname{rank}(\Delta^\star)\leq r$, so $\Delta^\star$ is feasible for the original problem. Moreover,
    \begin{align*}
        \Delta^\star\Psi
        &=\widetilde U_{\mathcal C}\left(\Lambda_{\mathcal C}^{-1/2}-I\right)\widetilde\Sigma_{\mathcal C}\widetilde V_{\mathcal C}^\top\\
        &=\widetilde U_{\mathcal C}\operatorname{Diag}\left(1-\widetilde\sigma_i\right)_{i\in\mathcal C}\widetilde V_{\mathcal C}^\top\\
        &=Z^\star,
    \end{align*}
    Thus, $\Delta^\star$ attains the optimal rank-$r$ approximation obtained from the relaxed problem. Together with the preceding lower bound, this gives
    $$
    \min_{\operatorname{rank}(\Delta)\leq r}\|\Delta\Psi-\Delta_{\mathcal A}\Psi\|_F^2
    =
    \min_{\operatorname{rank}(Z)\leq r}\|Z-\Delta_{\mathcal A}\Psi\|_F^2.
    $$
    Therefore, $\Delta^\star$ is a solution to Equation~\ref{eq:optimal-rank-constrained-refinement}.
    The resulting preconditioner and refined update are
    $$
    P=I_m+\Delta^\star,\quad T=P\Psi,
    $$
    where $\Delta^\star:=\widetilde U_{\mathcal C}\left(\Lambda_{\mathcal C}^{-1/2}-I\right)\widetilde U_{\mathcal C}^\top.$
\end{proof}

%% file: Appendix/appendix_convergence.tex
\subsection{Convergence Analysis of NormPre}\label{app:proof-normpre-convergence}
This section establishes convergence guarantees for \methodG and \methodL as stated in Theorem~\ref{thm:normpre-convergence}, with the Sketch-based implementation for \methodL given in Corollary~\ref{cor:normpre-sketch-convergence}.
Throughout this convergence analysis, we use a fixed orientation
$k_t\equiv k$ for both variants. For \methodG, we use the exact
matrix-sign transformation and assume $\Gamma_t\succ0$ for all $t$.

In the deterministic setting without momentum, let $G_t:=\nabla L(W_t)$ and define the actively oriented parameter and gradient as $\overline W_t:=\orient_{k_t}(W_t)$ and $\overline G_t:=\orient_{k_t}(G_t)$. For their $i$-th rows $\bar w_{t,i}$ and $\bar g_{t,i}$, respectively, define the relaxed tangent gradient
$$
x_{t,i}:=\bar g_{t,i}-\langle \bar g_{t,i},\bar w_{t,i}\rangle\bar w_{t,i}.
$$
Decompose $\bar g_{t,i}$ into the tangent and radial components with respect to $\bar w_{t,i}$:
$$
\bar g_{t,i}=g_{t,i}^{\mathrm{tan}}+g_{t,i}^{\mathrm{rad}},\quad g_{t,i}^{\mathrm{rad}}:=\frac{\langle \bar g_{t,i},\bar w_{t,i}\rangle}{\|\bar w_{t,i}\|_2^2}\bar w_{t,i}.
$$
Then $x_{t,i}=g_{t,i}^{\mathrm{tan}}+\left(1-\|\bar w_{t,i}\|_2^2\right)g_{t,i}^{\mathrm{rad}}.$
Thus, $x_{t,i}$ preserves the tangent component of $\bar g_{t,i}$ while retaining a norm-dependent radial component. In particular, it reduces to the exact tangent projection when $\|\bar w_{t,i}\|_2=1$. Collecting the resulting rows gives $X_t:=\mathcal T_{W_t,k_t}(G_t)$.

As in Section~\ref{sec:prelim}, we present the derivation for matrices in $\mathbb{R}^{m\times n}$ under the row orientation; the column-oriented case follows by applying the same argument to the transpose. We assume throughout that the row normalization is well defined. 
Through row normalization of $X_t$, we get the normalized update $\Psi_t:=D(X_t)^{-1}X_t$ and its $i$-th row $\psi_{t,i}$. 
Applying spectral preconditioning, we get the refined update $T_t$. As in Algorithm~\ref{alg:normpre}, the resulting chain is
$$
\overline G_t\xrightarrow{\text{Relaxed Tangent Projection}}X_t\xrightarrow{\text{Diagonal-Gram Normalization}}\Psi_t\xrightarrow{\text{Spectral Preconditioning}}T_t.
$$

The following lemmas quantify the gradient contribution retained by the normalized update, the spectral properties of the global and localized refinements and the radial interaction induced by each preconditioner. These quantities are used to lower bound the gradient-update inner product and control the smoothness term in the subsequent descent analyses.

\subsubsection{Lemmas}

\begin{lemma}
    \label{lem:tangent-base-alignment}
    Let $\phi_{t,i}$ denote the angle between $\bar g_{t,i}$ and $\bar w_{t,i}$. Suppose $\sin(\phi_{t,i})\geq\gamma>0$ for every row~$i$, where $\gamma$ represents the tangential fraction. With the diagonal-Gram normalization $\Psi_t:=D(X_t)^{-1}X_t$ in Algorithm~\ref{alg:normpre},
    $$
    \langle X_t,\Psi_t\rangle_F
    \geq \gamma\|\overline G_t\|_F.
    $$
\end{lemma}

\begin{proof}
    Since $x_{t,i}$ preserves the tangent component of $\bar g_{t,i}$ while retaining a norm-dependent radial component along $\bar w_{t,i}$, letting $\phi_{t,i}$ denote the angle between $\bar g_{t,i}$ and $\bar w_{t,i}$ gives
    $$
    \|x_{t,i}\|_2^2=\|\bar g_{t,i}\|_2^2\left[\sin^2(\phi_{t,i})+\left(1-\|\bar w_{t,i}\|_2^2\right)^2\cos^2(\phi_{t,i})\right]\geq\|\bar g_{t,i}\|_2^2\sin^2(\phi_{t,i}).
    $$
    Thus, we get $\|x_{t,i}\|_2\geq\|\bar g_{t,i}\|_2\sin(\phi_{t,i})$.
    Moreover, since $\Psi_t=D(X_t)^{-1}X_t$, its $i$-th row satisfies $\psi_{t,i}=x_{t,i}/\|x_{t,i}\|_2$. Then
    \begin{align*}
        \langle X_t,\Psi_t\rangle_F
        &=\sum_i\langle x_{t,i},\psi_{t,i}\rangle
        =\sum_i\left\langle x_{t,i},\frac{x_{t,i}}{\|x_{t,i}\|_2}\right\rangle
        =\sum_i\|x_{t,i}\|_2
        \geq\sum_i\|\bar g_{t,i}\|_2\sin(\phi_{t,i}).
    \end{align*}
    The lower bound becomes zero when the full gradient vanishes or when the gradient becomes parallel to the corresponding weight direction, $\sin(\phi_{t,i})=0$. If we assume that the gradient is not perfectly aligned with the weight, \textit{i.e.}, the gradient maintains a tangential fraction $\sin(\phi_{t,i})\geq\gamma>0$, we obtain
    \begin{align*}
        \langle X_t,\Psi_t\rangle_F
        \geq\gamma\sum_i\|\bar g_{t,i}\|_2
        \geq\gamma\left(\sum_i\|\bar g_{t,i}\|_2^2\right)^{1/2}
        =\gamma\|\overline G_t\|_F.
    \end{align*}
    Thus, this lemma quantifies the first-order gradient signal retained through the relaxed tangent projection and row normalization.
\end{proof}

\begin{lemma}
    \label{lem:exact-spectral-refinement}
    Consider the \methodG update with matrix-sign transformation and the \methodL update with Exact implementation in Algorithm~\ref{alg:normpre} and Section~\ref{sec:spectral-implementation}. For \methodG, suppose $\Gamma_t=\Psi_t\Psi_t^\top\succ0$. Let $P_t$ denote the corresponding spectral preconditioner and $T_t=P_t\Psi_t$. Then
    \begin{enumerate}[label=\textit{(\alph*)}, itemsep=0.3em, parsep=0em, topsep=0.3em, partopsep=0em, leftmargin=2em]
        \item Inner product with relaxed tangent gradient: $\langle X_t,T_t\rangle_F\geq\lambda_{\min}(P_t)\langle X_t,\Psi_t\rangle_F$.
        
        \item Norm of the refined update: $\|T_t\|_F\leq\max\{m,n\}^{1/2}$.

        \item Minimum eigenvalue of preconditioner: $\lambda_{\min}(P_t)\geq\lambda_{t,1}^{-1/2}\geq\max\{m,n\}^{-1/2}$, where the first inequality holds with equality for \methodG.
    \end{enumerate}
\end{lemma}

\begin{proof}
    Recall that $X_t$ is the relaxed tangent gradient under the active orientation, $\Psi_t=D(X_t)^{-1}X_t$ is the normalized update and $T_t=P_t\Psi_t$ is the update after spectral preconditioning. We follow the chain
    $$
    \overline{G}_t \xrightarrow{\text{Relaxed Tangent Projection}} X_t \xrightarrow{\text{Diagonal-Gram Normalization}} \Psi_t \xrightarrow{\text{Spectral Preconditioning}} T_t.
    $$

    \textbf{We first characterize how spectral preconditioning changes the inner product with the relaxed tangent gradient.}
    Let $\Gamma_t=\Psi_t\Psi_t^\top=\widetilde U_t\Lambda_t\widetilde U_t^\top$, and let $\tilde u_{t,i}$ denote the $i$-th column of $\widetilde U_t$, with corresponding eigenvalue $\lambda_{t,i}$ of $\Gamma_t$.
    For \methodG, $\Gamma_t\succ0$ and
    $$
    P_t=\Gamma_t^{-1/2}=\widetilde U_t\Lambda_t^{-1/2}\widetilde U_t^\top.
    $$
    We have $P_t\tilde u_{t,i}=\lambda_{t,i}^{-1/2}\tilde u_{t,i}$ for every $i$.
    For \methodL,
    $$
    P_t=I+\widetilde U_{\mathcal C_t}(\Lambda_{\mathcal C_t}^{-1/2}-I)\widetilde U_{\mathcal C_t}^\top.
    $$
    For $i\in\mathcal C_t$, $P_t\tilde u_{t,i}=\lambda_{t,i}^{-1/2}\tilde u_{t,i},$ while for $i\notin\mathcal C_t$, $P_t\tilde u_{t,i}=\tilde u_{t,i}$. Since every selected mode satisfies $\lambda_{t,i}>1$, the eigenvalues of $P_t$ are $\lambda_{t,i}^{-1/2}\in(0,1)$ on the selected modes and $1$ on the unselected modes.

    Therefore, for both realizations, $P_t$ and $\Gamma_t$ share the same eigenbasis and $P_t\succ0$. Since $P_t\succeq\lambda_{\min}(P_t)I$,
    \begin{align}\label{eq:lemma2_pt_gammat_lambdaminp}
        P_t\Gamma_t\succeq\lambda_{\min}(P_t)\Gamma_t.
    \end{align}
    With $\Psi_t=D(X_t)^{-1}X_t$, let $D_t:=D(X_t)=\operatorname{Diag}\left(\|x_{t,1}\|_2,\ldots,\|x_{t,m}\|_2\right).$
    Then $X_t=D_t\Psi_t$ and $D_t\succeq0$. Using $T_t=P_t\Psi_t$ and $\Gamma_t=\Psi_t\Psi_t^\top$,
    \begin{align}
        \langle X_t,T_t\rangle_F
        &=\operatorname{tr}\!\left(X_t^\top T_t\right)
        =\operatorname{tr}\!\left((D_t\Psi_t)^\top P_t\Psi_t\right)
        \nonumber \\
        &=\operatorname{tr}\!\left(D_tP_t\Psi_t\Psi_t^\top\right)
        =\operatorname{tr}(D_tP_t\Gamma_t).\label{eq:lemma2_xt_Tt_Fnorm}
    \end{align}
    Together with Equation~\ref{eq:lemma2_pt_gammat_lambdaminp}, we have $\operatorname{tr}\!\left(D_t\left(P_t\Gamma_t-\lambda_{\min}(P_t)\Gamma_t\right)\right)\geq0,$ and thus 
    \begin{align}\label{eq:lemma2_dt_pt_gammat}
        \operatorname{tr}(D_tP_t\Gamma_t)\geq\lambda_{\min}(P_t)\operatorname{tr}(D_t\Gamma_t).
    \end{align}
    Since each row of $\Psi_t$ has unit norm, $\operatorname{diag}(\Gamma_t)=\mathbf 1$. Then
    \begin{align}\label{eq:lemma2_tr_dt_gammat}
        \operatorname{tr}(D_t\Gamma_t)
        =\sum_i(D_t)_{ii}(\Gamma_t)_{ii}
        =\sum_i\|x_{t,i}\|_2.
    \end{align}
    Combining Equations~\ref{eq:lemma2_xt_Tt_Fnorm}--\ref{eq:lemma2_tr_dt_gammat} gives $\langle X_t,T_t\rangle_F\geq\lambda_{\min}(P_t)\sum_i\|x_{t,i}\|_2.$
    Since $\psi_{t,i}=x_{t,i}/\|x_{t,i}\|_2$,
    $$
    \langle X_t,\Psi_t\rangle_F=\sum_i\langle x_{t,i},\psi_{t,i}\rangle=\sum_i\|x_{t,i}\|_2.
    $$
    Therefore,
    \begin{align}\label{eq:lemma2_conclusion1}
        \langle X_t,T_t\rangle_F\geq\lambda_{\min}(P_t)\langle X_t,\Psi_t\rangle_F.
    \end{align}
    This inequality quantifies the change in the inner product with the relaxed tangent gradient induced by exact spectral preconditioning: $\langle X_t,\Psi_t\rangle_F$ before preconditioning is transformed into $\langle X_t,T_t\rangle_F$ after preconditioning, while retaining at least a $\lambda_{\min}(P_t)$ fraction. Together with Lemma~\ref{lem:tangent-base-alignment}, this provides the positive first-order descent signal used later in the descent analysis.
    
    \textbf{We next bound the norm of the refined update $T_t$.}
    Since $P_t$ is symmetric and $T_t=P_t\Psi_t$,
    $$
    T_tT_t^\top=P_t\Psi_t\Psi_t^\top P_t=P_t\Gamma_tP_t.
    $$
    Recall that $P_t$ and $\Gamma_t$ share the same eigenbasis, and let $p_{t,i}$ denote the eigenvalue of $P_t$ associated with $\tilde u_{t,i}$. Then $T_tT_t^\top=\widetilde U_t\operatorname{Diag}(p_{t,i}^2\lambda_{t,i})\widetilde U_t^\top.$

    For \methodG, $p_{t,i}=\lambda_{t,i}^{-1/2}$ for every $i$. We have $p_{t,i}^2\lambda_{t,i}=1$ and
    $$
    \|T_t\|_F^2=\operatorname{tr}(T_tT_t^\top)=m.
    $$
    
    For \methodL, $p_{t,i}=\lambda_{t,i}^{-1/2}$ for $i\in\mathcal C_t$ and $p_{t,i}=1$ for $i\notin\mathcal C_t$. We have
    $$
    \|T_t\|_F^2=\operatorname{tr}(T_tT_t^\top)=\sum_{i\in\mathcal C_t}1+\sum_{i\notin\mathcal C_t}\lambda_{t,i}\leq\sum_i\lambda_{t,i}=\operatorname{tr}(\Gamma_t)=m.
    $$
    The inequality follows because every selected mode satisfies $\lambda_{t,i}>1$. Therefore, for both realizations,
    $$
    \|T_t\|_F\leq m^{1/2}\leq\max\{m,n\}^{1/2},
    $$
    which later controls the quadratic term in the smoothness inequality.

    \textbf{Finally, we bound the minimum eigenvalue of the preconditioner.}
    For \methodG, the eigenvalues of $P_t=\Gamma_t^{-1/2}$ are $\lambda_{t,i}^{-1/2}$. We have $\lambda_{\min}(P_t)=\lambda_{t,1}^{-1/2}.$
    For \methodL, the eigenvalues of $P_t$ are $\lambda_{t,i}^{-1/2}$ on the selected modes and $1$ on the unselected modes. Since $\lambda_{t,i}\leq\lambda_{t,1}$ for every selected mode and $\lambda_{t,1}\geq1$, every eigenvalue of $P_t$ is at least $\lambda_{t,1}^{-1/2}$.
    Therefore, for both realizations,
    $$
    \lambda_{\min}(P_t)\geq\lambda_{t,1}^{-1/2}.
    $$
    Since $\operatorname{tr}(\Gamma_t)=m$, we have $\lambda_{t,1}\leq m$. Thus,
    $$
    \lambda_{\min}(P_t)\geq\lambda_{t,1}^{-1/2}\geq m^{-1/2}\geq\max\{m,n\}^{-1/2}.
    $$
\end{proof}

\begin{lemma}
    \label{lem:cross-row-radial-leakage}
    Under the setting of Lemma~\ref{lem:exact-spectral-refinement}, let $\nu_t:=\max_i\frac{\|\bar g_{t,i}-x_{t,i}\|_2}{\|x_{t,i}\|_2}$ denote the maximum radial ratio. Let $\phi'_{t,j,i}$ denote the angle between $\psi_{t,j}$ and $\bar w_{t,i}$, and suppose $\max_i(\sum_j\cos^2(\phi'_{t,j,i}))^{1/2}\leq\gamma'$. For \methodG, let $\kappa_t:=\sqrt{\lambda_{t,1}/\lambda_{t,m}}$ and define
    $$
    \chi_t:=\begin{cases}\kappa_t,&\methodG,\\ \sqrt{\lambda_{t,1}},&\methodL.\end{cases}
    $$
    Then, we have
    $$
    \left|\langle \overline G_t-X_t,T_t\rangle_F\right|\leq\nu_t\gamma'\chi_t\langle X_t,T_t\rangle_F.
    $$
\end{lemma}

\begin{proof}
    Recall that $X_t$ is the rowwise relaxed tangent gradient of $\overline G_t$ and $\overline G_t-X_t$ is the corresponding radial residual. Using $T_t=P_t\Psi_t$, we have
    $$
    \langle \overline G_t-X_t,T_t\rangle_F=\langle \overline G_t-X_t,P_t\Psi_t\rangle_F.
    $$
    By the definition of $\nu_t$,
    $$
    \frac{\|\bar g_{t,i}-x_{t,i}\|_2}{\|x_{t,i}\|_2}\leq\nu_t,
    $$
    or equivalently, $\|\bar g_{t,i}-x_{t,i}\|_2\leq\nu_t\|x_{t,i}\|_2$.

    We next bound the interaction between the radial residual and the spectrally refined update.
    From the definition of the relaxed tangent projection, $\bar g_{t,i}-x_{t,i}=\langle\bar g_{t,i},\bar w_{t,i}\rangle\bar w_{t,i}$. Let $e_i$ denote the $i$-th standard basis vector. Expanding the Frobenius inner product rowwise gives
    \begin{align*}
        \left|\langle \overline G_t-X_t,T_t\rangle_F\right|
        &=\left|\langle \overline G_t-X_t,P_t\Psi_t\rangle_F\right|
        =\left|\sum_i\left\langle \bar g_{t,i}-x_{t,i},e_i^\top P_t\Psi_t\right\rangle\right|\\
        &=\left|\sum_i\langle\bar g_{t,i},\bar w_{t,i}\rangle e_i^\top P_t\Psi_t\bar w_{t,i}^\top\right|
        \leq\sum_i\|\bar g_{t,i}-x_{t,i}\|_2\|e_i^\top P_t\|_2\frac{\|\Psi_t\bar w_{t,i}^\top\|_2}{\|\bar w_{t,i}\|_2}\\
        &\leq\nu_t\|P_t\|_{\mathrm{op}}\sum_i\|x_{t,i}\|_2\frac{\|\Psi_t\bar w_{t,i}^\top\|_2}{\|\bar w_{t,i}\|_2}.
    \end{align*}
    For all $j$ and $i$, let $\phi'_{t,j,i}$ denote the angle between $\psi_{t,j}$ and $\bar w_{t,i}$. Since each row of $\Psi_t$ has unit norm,
    $$
    \frac{\|\Psi_t\bar w_{t,i}^\top\|_2}{\|\bar w_{t,i}\|_2}=\left(\sum_j\cos^2(\phi'_{t,j,i})\right)^{1/2}\leq\gamma'.
    $$
    This quantity measures the aggregate non-orthogonality between the $i$-th weight direction and the normalized relaxed tangent gradients. A smaller $\gamma'$ corresponds to stronger orthogonality between the normalized relaxed tangent gradients and the weight directions. Therefore,
    \begin{align}\label{eq:lemma3_radial_refinement}
        \left|\langle \overline G_t-X_t,T_t\rangle_F\right|\leq\nu_t\gamma'\|P_t\|_{\mathrm{op}}\sum_i\|x_{t,i}\|_2.
    \end{align}
    From Lemma~\ref{lem:exact-spectral-refinement}(a), $\langle X_t,T_t\rangle_F\geq\lambda_{\min}(P_t)\sum_i\|x_{t,i}\|_2.$
    Combining this inequality with Equation~\ref{eq:lemma3_radial_refinement} gives
    $$
    \left|\langle \overline G_t-X_t,T_t\rangle_F\right|\leq\nu_t\gamma'\frac{\|P_t\|_{\mathrm{op}}}{\lambda_{\min}(P_t)}\langle X_t,T_t\rangle_F.
    $$

    For \methodG, $P_t=\Gamma_t^{-1/2}$ and the eigenvalues of $\Gamma_t$ satisfy $\lambda_{t,1}\geq\cdots\geq\lambda_{t,m}>0$. We have
    $$
    \|P_t\|_{\mathrm{op}}=\lambda_{t,m}^{-1/2},\quad \lambda_{\min}(P_t)=\lambda_{t,1}^{-1/2}.
    $$
    Therefore,
    $$
    \frac{\|P_t\|_{\mathrm{op}}}{\lambda_{\min}(P_t)}=\sqrt{\frac{\lambda_{t,1}}{\lambda_{t,m}}}=\kappa_t.
    $$
    For \methodL, all eigenvalues of $P_t$ lie in $(0,1]$, so $\|P_t\|_{\mathrm{op}}\leq1$. From Lemma~\ref{lem:exact-spectral-refinement}(c), $\lambda_{\min}(P_t)\geq\lambda_{t,1}^{-1/2}$. Therefore,
    $$
    \frac{\|P_t\|_{\mathrm{op}}}{\lambda_{\min}(P_t)}\leq\sqrt{\lambda_{t,1}}.
    $$
    Combining the two cases gives
    $$
    \left|\langle \overline G_t-X_t,T_t\rangle_F\right|\leq\nu_t\gamma'\chi_t\langle X_t,T_t\rangle_F.
    $$
    In summary, the radial interaction in the spectrally refined update is controlled relative to the inner product between the refined update and the relaxed tangent gradient. Its relative magnitude is determined by the radial ratio $\nu_t$, the non-orthogonality $\gamma'$ between the normalized row $\psi_{t,j}$ and the weight $\bar w_{t,i}$ and the spectral factor $\chi_t$.
\end{proof}

\subsubsection{Convergence of \methodG and \methodL (Exact)}\label{app:proof-exact}

We now combine Lemmas~\ref{lem:tangent-base-alignment}, \ref{lem:exact-spectral-refinement} and~\ref{lem:cross-row-radial-leakage} to prove the \methodG and \methodL (Exact) cases of Theorem~\ref{thm:normpre-convergence}.

\begin{proof}
    We consider \methodG with matrix-sign transformation and \methodL with Exact implementation in the deterministic setting without momentum, RMS scaling or weight decay. Let $\chi_t=\kappa_t$ for \methodG and $\chi_t=\sqrt{\lambda_{t,1}}$ for \methodL, as in Lemma~\ref{lem:cross-row-radial-leakage}. Taking $G_t=\nabla \mathcal L(W_t)$, the final update is
    $$
    W_{t+1}=W_t-\eta\,\orient_{k_t}^{-1}(T_t).
    $$
    Since $\orient_{k_t}$ is either the identity or transpose, it preserves Frobenius inner products and norms. Hence, $\langle G_t,\orient_{k_t}^{-1}(T_t)\rangle_F=\langle \overline G_t,T_t\rangle_F$ and $\|\orient_{k_t}^{-1}(T_t)\|_F=\|T_t\|_F$. Moreover, $\|\overline G_t\|_F=\|G_t\|_F=\|\nabla \mathcal L(W_t)\|_F$.
    
    By the $L$-smoothness of $\mathcal L$,
    \begin{align}
        \mathcal L(W_{t+1})&\leq\mathcal L(W_t)+\langle \nabla\mathcal L(W_t),W_{t+1}-W_t\rangle_F+\frac{L}{2}\|W_{t+1}-W_t\|_F^2\nonumber\\
        &= \mathcal L(W_t)-\eta\left\langle G_t,\orient_{k_t}^{-1}(T_t)\right\rangle_F+\frac{L\eta^2}{2}\left\|\orient_{k_t}^{-1}(T_t)\right\|_F^2 \nonumber\\
        &=\mathcal L(W_t)-\eta\langle \overline G_t,T_t\rangle_F+\frac{L\eta^2}{2}\|T_t\|_F^2.
        \label{eq:exact-smoothness}
    \end{align}
    Thus, the descent analysis reduces to lower bounding the first-order gradient signal $\langle \overline G_t,T_t\rangle_F$ and upper bounding the refined update norm $\|T_t\|_F$.

    We first lower bound $\langle \overline G_t,T_t\rangle_F$.
    Using the decomposition $\overline G_t=X_t+(\overline G_t-X_t)$,
    \begin{align*}
        \langle \overline G_t,T_t\rangle_F
        &=\langle X_t,T_t\rangle_F+\langle \overline G_t-X_t,T_t\rangle_F
        \geq\left(1-\nu_t\gamma'\chi_t\right)\langle X_t,T_t\rangle_F\\
        &\geq\epsilon\langle X_t,T_t\rangle_F
        \geq\epsilon\lambda_{\min}(P_t)\langle X_t,\Psi_t\rangle_F\\
        &\geq\epsilon\gamma\lambda_{\min}(P_t)\|\overline G_t\|_F
        \geq\epsilon\gamma\max\{m,n\}^{-1/2}\|\overline G_t\|_F.
    \end{align*}
    We now explain each step in this bound. The first inequality follows from Lemma~\ref{lem:cross-row-radial-leakage}, which gives
    $$
    \left|\langle \overline G_t-X_t,T_t\rangle_F\right|\leq\nu_t\gamma'\chi_t\langle X_t,T_t\rangle_F.
    $$
    The second inequality follows from the assumption that there exists $\epsilon>0$ such that
    $$
    \max_{0\leq t\leq T}\nu_t\gamma'\chi_t\leq1-\epsilon.
    $$
    This condition ensures that, in the worst case, the radial interaction can cancel at most a $(1-\epsilon)$ fraction of the positive first-order term $\langle X_t,T_t\rangle_F$, leaving at least an $\epsilon$ fraction in the full gradient-update inner product. Since Lemma~\ref{lem:cross-row-radial-leakage} bounds the sign-indefinite term $\langle \overline G_t-X_t,T_t\rangle_F$ only in magnitude, this provides a sufficient but not necessary condition for convergence.

    The third inequality follows from Lemma~\ref{lem:exact-spectral-refinement}(a), which gives $\langle X_t,T_t\rangle_F\geq\lambda_{\min}(P_t)\langle X_t,\Psi_t\rangle_F.$
    The fourth inequality follows from Lemma~\ref{lem:tangent-base-alignment}, which gives $\langle X_t,\Psi_t\rangle_F\geq\gamma\|\overline G_t\|_F$.
    Finally, Lemma~\ref{lem:exact-spectral-refinement}(c) gives $\lambda_{\min}(P_t)\geq\max\{m,n\}^{-1/2}$.

    We next control the quadratic term in Equation~\ref{eq:exact-smoothness}. Lemma~\ref{lem:exact-spectral-refinement}(b) gives $\|T_t\|_F^2\leq\max\{m,n\}.$
    Substituting the preceding lower and upper bounds into Equation~\ref{eq:exact-smoothness} and using $\|\overline G_t\|_F=\|\nabla \mathcal L(W_t)\|_F$, gives
    \begin{align}
        \mathcal L(W_{t+1})
        &\leq \mathcal L(W_t)-\eta\epsilon\gamma\max\{m,n\}^{-1/2}\|\nabla \mathcal L(W_t)\|_F+\frac{L\eta^2}{2}\max\{m,n\}.
        \label{eq:exact-one-step-descent}
    \end{align}

    Summing Equation~\ref{eq:exact-one-step-descent} from $t=0$ to $T$ gives
    $$
    \eta\epsilon\gamma\max\{m,n\}^{-1/2}\sum_{t=0}^{T}\|\nabla \mathcal L(W_t)\|_F\leq\mathcal L(W_0)-\mathcal L(W_{T+1})+\frac{(T+1)L\eta^2}{2}\max\{m,n\}.
    $$
    Since $\mathcal L(W_{T+1})\geq\mathcal L_{\inf}$ and $\Delta_0:=\mathcal L(W_0)-\mathcal L_{\inf}$,
    $$
    \eta\epsilon\gamma\max\{m,n\}^{-1/2}\sum_{t=0}^{T}\|\nabla \mathcal L(W_t)\|_F\leq\Delta_0+\frac{(T+1)L\eta^2}{2}\max\{m,n\}.
    $$
    Dividing by $\eta\epsilon\gamma\max\{m,n\}^{-1/2}(T+1)$ and using that the minimum does not exceed the average yields
    $$
    \min_{0\leq t\leq T}\|\nabla \mathcal L(W_t)\|_F\leq\frac{\Delta_0\max\{m,n\}^{1/2}}{\eta\epsilon\gamma(T+1)}+\frac{L\eta\max\{m,n\}^{3/2}}{2\epsilon\gamma}.
    $$
    For any constant $C>0$, choosing $\eta=C/\sqrt{T+1}$ gives
    $$
    \min_{0\leq t\leq T}\|\nabla \mathcal L(W_t)\|_F\leq\frac{1}{\sqrt{T+1}}\left(\frac{\Delta_0\max\{m,n\}^{1/2}}{\epsilon\gamma C}+\frac{LC\max\{m,n\}^{3/2}}{2\epsilon\gamma}\right).
    $$
    For $\Delta_0>0$, the bound is minimized by
    $$
    C^\star=\sqrt{\frac{2\Delta_0}{L\max\{m,n\}}},\qquad \eta^\star=\frac{C^\star}{\sqrt{T+1}}=\sqrt{\frac{2\Delta_0}{L\max\{m,n\}(T+1)}}.
    $$
    Substituting $C^\star$ gives
    $$
    \min_{0\leq t\leq T}\|\nabla \mathcal L(W_t)\|_F\leq\frac{\max\{m,n\}\sqrt{2L\Delta_0}}{\epsilon\gamma\sqrt{T+1}}.
    $$
    Thus, both \methodG with matrix-sign transformation and \methodL with Exact implementation achieve an $\mathcal O(T^{-1/2})$ convergence rate under their corresponding conditions.
\end{proof}

\subsubsection{Convergence of \methodL (Sketch)}\label{app:proof-sketch}

We next extend the preceding analysis to \methodL with the Sketch-based implementation by viewing the sketched localized update as a perturbation of the exact update constructed from the same normalized update $\Psi_t$.

\begin{corollary}[Convergence of \methodL (Sketch) w/o Momentum]
    \label{cor:normpre-sketch-convergence}
    Suppose that the assumptions of Theorem~\ref{thm:normpre-convergence} hold for the sequence generated by \methodL with the Sketch-based implementation. Let $P_t$ and $\widehat P_t$ denote the exact and sketch-based localized spectral preconditioners constructed from the same $\Psi_t$, respectively. Assume that there exists $\delta\geq0$ such that $\|\widehat P_t-P_t\|_{\mathrm{op}}\leq\delta$ and $\delta\max_{0\leq t\leq T}\lambda_{t,1}<\epsilon\gamma$. For any constant $C>0$, choosing $\eta=C/\sqrt{T+1}$ gives
    $$
    \min_{0\leq t\leq T}\|\nabla\mathcal L(W_t)\|_F\leq\frac{1}{\sqrt{T+1}}\left(\frac{\Delta_0\max\{m,n\}^{1/2}}{\left(\epsilon\gamma-\delta\max_{0\leq t\leq T}\lambda_{t,1}\right)C}+\frac{LC\max\{m,n\}^{3/2}}{2\left(\epsilon\gamma-\delta\max_{0\leq t\leq T}\lambda_{t,1}\right)}\right).
    $$
\end{corollary}

\begin{proof}
    Let $P_t$ and $\widehat P_t$ denote the localized spectral preconditioners constructed from the exact and sketch-based interaction eigenspaces, respectively (Section~\ref{sec:spectral-implementation}):
    $$
    P_t:=I+\widetilde U_{\mathcal C_t}\left(\Lambda_{\mathcal C_t}^{-1/2}-I\right)\widetilde U_{\mathcal C_t}^{\top},\quad \widehat P_t:=I+\widehat U_{\mathcal C_t}\left(\widehat\Lambda_{\mathcal C_t}^{-1/2}-I\right)\widehat U_{\mathcal C_t}^{\top}.
    $$
    For notational simplicity, $\mathcal C_t$ denotes the selected index set in each implementation; the exact and sketch-based sets may differ in both membership and cardinality.
    The corresponding updates are $T_t:=P_t\Psi_t$ and $\widehat T_t:=\widehat P_t\Psi_t$. Both are evaluated from the same $\Psi_t$ and differ only in the interaction eigenspace extraction.

    We first bound the resulting update perturbation. Using the assumption $\|\widehat P_t-P_t\|_{\mathrm{op}}\leq\delta$, we have
    \begin{align*}
        \|\widehat T_t-T_t\|_{\mathrm{op}}
        =\|(\widehat P_t-P_t)\Psi_t\|_{\mathrm{op}}
        \leq\|\widehat P_t-P_t\|_{\mathrm{op}}\|\Psi_t\|_{\mathrm{op}}
        \leq\delta\sqrt{\lambda_{t,1}},
    \end{align*}
    where $\|\Psi_t\|_{\mathrm{op}}=\sqrt{\lambda_{t,1}}$ follows from $\Gamma_t=\Psi_t\Psi_t^\top$.

    We next lower bound the first-order gradient signal. By nuclear/operator norm duality,
    \begin{align}\label{eq:coro-g-T-sketch}
        \langle\overline G_t,\widehat T_t\rangle_F
        &\geq\langle\overline G_t,T_t\rangle_F-\|\overline G_t\|_*\|\widehat T_t-T_t\|_{\mathrm{op}}.
    \end{align}
    As in Lemma~\ref{lem:tangent-base-alignment}, $\|x_{t,i}\|_2\geq\|\bar g_{t,i}\|_2\sin(\phi_{t,i})$ with $\sin(\phi_{t,i})\geq\gamma$. Then, $\|\overline G_t\|_*\leq\sum_i\|\bar g_{t,i}\|_2\leq\gamma^{-1}\sum_i\|x_{t,i}\|_2=\gamma^{-1}\langle X_t,\Psi_t\rangle_F$. Therefore,
    \begin{align}\label{eq:coro-g-star}
        \|\overline G_t\|_*\|\widehat T_t-T_t\|_{\mathrm{op}}\leq\frac{\delta\sqrt{\lambda_{t,1}}}{\gamma}\langle X_t,\Psi_t\rangle_F.
    \end{align}
    As derived in Appendix~\ref{app:proof-exact} for \methodL (Exact),
    \begin{align}\label{eq:coro-g-T-exact}
        \langle\overline G_t,T_t\rangle_F\geq\epsilon\lambda_{\min}(P_t)\langle X_t,\Psi_t\rangle_F.
    \end{align}
    Combining Equations~\ref{eq:coro-g-T-sketch}--\ref{eq:coro-g-T-exact} with $\lambda_{\min}(P_t)\geq\lambda_{t,1}^{-1/2}$ from Lemma~\ref{lem:exact-spectral-refinement}, we obtain
    \begin{align*}
        \langle\overline G_t,\widehat T_t\rangle_F
        &\geq\left(\epsilon\lambda_{t,1}^{-1/2}-\frac{\delta\sqrt{\lambda_{t,1}}}{\gamma}\right)\langle X_t,\Psi_t\rangle_F\\
        &=\frac{\epsilon\gamma-\delta\lambda_{t,1}}{\gamma\sqrt{\lambda_{t,1}}}\langle X_t,\Psi_t\rangle_F\\
        &\geq\frac{\epsilon\gamma-\delta\lambda_{t,1}}{\sqrt{\lambda_{t,1}}}\|\overline G_t\|_F\\
        &\geq\left(\epsilon\gamma-\delta\max_{0\leq t\leq T}\lambda_{t,1}\right)\max\{m,n\}^{-1/2}\|\overline G_t\|_F.
    \end{align*}
    Thus, the assumption $\delta\max_{0\leq t\leq T}\lambda_{t,1}<\epsilon\gamma$ preserves a positive first-order gradient signal.

    It remains to bound the sketched update norm. From $\widehat P_t=I+\widehat U_{\mathcal C_t}\left(\widehat\Lambda_{\mathcal C_t}^{-1/2}-I\right)\widehat U_{\mathcal C_t}^\top$, since the columns of $\widehat U_{\mathcal C_t}$ are orthonormal, $\widehat P_t$ has eigenvalues $\widehat\lambda_{t,i}^{-1/2}$ along the selected directions and $1$ on the non-selected ones. Since the selected modes satisfy $\widehat\lambda_{t,i}>1$, all eigenvalues of $\widehat P_t$ lie in $(0,1]$. Therefore, $\|\widehat T_t\|_F\leq\|\widehat P_t\|_{\mathrm{op}}\|\Psi_t\|_F\leq\|\Psi_t\|_F$, which gives $\|\widehat T_t\|_F^2\leq\max\{m,n\}$.

    Substituting the preceding first-order and norm bounds into the same $L$-smoothness derivation as in the proof of Appendix~\ref{app:proof-exact} gives
    $$
    \mathcal L(W_{t+1})\leq\mathcal L(W_t)-\eta\left(\epsilon\gamma-\delta\max_{0\leq t\leq T}\lambda_{t,1}\right)\max\{m,n\}^{-1/2}\|\nabla\mathcal L(W_t)\|_F+\frac{L\eta^2}{2}\max\{m,n\}.
    $$
    Summing from $t=0$ to $T$, using $\mathcal L(W_{T+1})\geq\mathcal L_{\inf}$ and $\Delta_0=\mathcal L(W_0)-\mathcal L_{\inf}$ and choosing $\eta=C/\sqrt{T+1}$ as in Theorem~\ref{thm:normpre-convergence} yields
    $$
    \min_{0\leq t\leq T}\|\nabla\mathcal L(W_t)\|_F\leq\frac{1}{\sqrt{T+1}}\left(\frac{\Delta_0\max\{m,n\}^{1/2}}{\left(\epsilon\gamma-\delta\max_{0\leq t\leq T}\lambda_{t,1}\right)C}+\frac{LC\max\{m,n\}^{3/2}}{2\left(\epsilon\gamma-\delta\max_{0\leq t\leq T}\lambda_{t,1}\right)}\right).
    $$
    Therefore, under the condition $\delta\max_{0\leq t\leq T}\lambda_{t,1}<\epsilon\gamma$ with a positive margin independent of $T$, \methodL with the Sketch-based implementation preserves the $\mathcal O(T^{-1/2})$ convergence rate established for the Exact implementation.
\end{proof}